\documentclass{article}
\usepackage[preprint]{neurips_2025}

\usepackage[utf8]{inputenc}
\usepackage[T1]{fontenc}
\usepackage[english]{babel}
\usepackage{amsmath, amssymb, amsfonts, amsthm, amsopn}
\usepackage{graphicx, subcaption, rotating}
\usepackage{algorithm}
\usepackage{setspace}
\usepackage{dsfont}
\usepackage{booktabs, multirow}
\usepackage{adjustbox}
\usepackage{listings}
\usepackage{xcolor}
\usepackage{enumitem}
\usepackage{anysize}
\usepackage{nicefrac}
\usepackage{microtype}
\usepackage{times}
\usepackage{url}

\usepackage[colorlinks=true, citecolor=blue, linkcolor=blue]{hyperref}

\newtheorem{Proposition}{Proposition}
\newtheorem{Definition}{Definition}
\newtheorem{Corollary}{Corollary}
\newtheorem{Theorem}{Theorem}
\newtheorem{Lemma}{Lemma}
\newtheorem{Remark}{Remark}
\newtheorem{Assumption}{Assumption}

\renewcommand{\d}{\mathrm{d}}

\DeclareMathOperator*{\argmin}{arg\,min}
\allowdisplaybreaks

\def\Assumptionautorefname{Assumption}

\title{Exponential Convergence of CAVI for Bayesian PCA}

\author{%
  Arghya Datta \\
  Department of Mathematics and Statistics \\
  Université de Montréal \\
  \And
  Philippe Gagnon \\
  Department of Mathematics and Statistics \\
  Université de Montréal \\
  \And
  Florian Maire \\
  Department of Mathematics and Statistics \\
  Université de Montréal \\
}

\begin{document}

\def\sectionautorefname{Section}
\def\subsectionautorefname{Section}
\def\subsubsectionautorefname{Section}
\def\Assumptionautorefname{Assumption}
\maketitle

\begin{abstract}
 Probabilistic principal component analysis (PCA) and its Bayesian variant (BPCA) are widely used for dimension reduction in machine learning and statistics. The main advantage of probabilistic PCA over the traditional formulation is allowing uncertainty quantification. The parameters of BPCA are typically learned using mean-field variational inference, and in particular, the coordinate ascent variational inference (CAVI) algorithm. So far, the convergence speed of CAVI for BPCA has not been characterized. In our paper, we fill this gap in the literature. Firstly, we prove a precise exponential convergence result in the case where the model uses a single principal component (PC). Interestingly, this result is established through a connection with the classical \textit{power iteration algorithm} and it indicates that traditional PCA is retrieved as points estimates of the BPCA parameters. Secondly, we leverage recent tools to prove exponential convergence of CAVI for the model with any number of PCs, thus leading to a more general result, but one that is of a slightly different flavor. To prove the latter result, we additionally needed to introduce a novel lower bound for the symmetric Kullback--Leibler divergence between two multivariate normal distributions, which, we believe, is of independent interest in information theory.
\end{abstract}
\vspace{0.5em}

\section{Introduction}
\subsection{Probabilistic principal component analysis}\label{section1.1}
Principal component analysis (PCA) is a ubiquitous dimension reduction tool. It has a broad spectrum of applications ranging from computer vision \citep{rousseau2024improving} to finance \citep{veronesi2016handbook,swanson2020predicting}. In \cite{veronesi2016handbook}, it is explained that PCA is commonly used in a context of interest rate modeling to produce a low dimensional input to the model. More precisely, PCA is applied to yield curve data and three principal components (PCs), roughly corresponding to the level, slope and curvature of the yield curves, serve as inputs to the model. When the PCs have an interesting interpretation and serve as inputs to a model like in this example, uncertainty quantification is of importance. However, PCA is not a statistical model and thus uncertainty quantification is not available. Probabilistic principal component analysis (PPCA) is a model introduced by \cite{tipping} to remedy the situation.  It has been widely used for uncertainty quantification in machine learning and statistics (see, e.g., \cite{ilin2010practical} and \cite{wu2021uncertainty}). The PPCA model is defined as:
\begin{equation}\label{ppca1}
x_i=\textbf{W}z_i+\tau^{-1/2}\epsilon_i,\quad i = 1, \ldots, n,
\end{equation}
where $\textbf{W} \in \mathbb{R}^{d \times k}$ is a loading matrix, $z_i \in \mathbb{R}^k$ a standard normal latent variable, $\tau>0$ a precision parameter, and $\epsilon_i \in \mathbb{R}^d$ a standardized normal error vector. Thus, $x_i \in\mathbb{R}^d$ plays the role of a centered data vector, and it is assumed that $n \geqslant d$. The columns of the matrix \textbf{Z}, with lines given by $z_1, \ldots, z_n$, play an analogous role as the PCs in PCA. The PPCA model is complex by, among others, its non-identifiability. 

A Bayesian version, henceforth referred to as BPCA, has been introduced in \cite{bishop1}. In this Bayesian version, $\textbf{W}$ has a prior given by a matrix normal distribution \citep{gupta1999matrix} with mean $\mathbf{0}$ and covariance parameter $I_d\otimes \Lambda^{-1}$, where $I_d$ is the identity matrix of size $d$ and $\Lambda$ is a diagonal matrix of precision parameters. This covariance structure implies that the rows of $\textbf{W}$ are independent with a shared covariance of $\Lambda^{-1}$. The prior distribution on $\tau$ is a gamma. 

Despite the widespread use of such models, their theoretical properties remain largely unexplored. Recently, however, there has been a growing interest in establishing rigorous guarantees. For example, the consistency of the maximum likelihood estimator has been proved in \cite{NEURIPS2023_5b0c0b2c}. On the Bayesian side, a selection procedure for the number of PCs has been proposed and established to be consistent in \cite{cherief2019consistency}. In this work, we establish exponential convergence of CAVI (coordinate ascent variational inference, \cite{bishop2006pattern}) for mean-field variational inference of BPCA. Variational inference is used when Bayesian inference using Markov chain Monte Carlo methods \citep{hastings1970monte,neal2011mcmc} are inefficient. In \cite{819772}, it has been highlighted to be the case for BPCA given the high-dimensional nature of the model (the number of unknowns is ${dk + nk}$). In our theoretical analyses, $\tau$ is assumed fixed and known to simplify, with a value equal to $\tau_0 > 0$.
\subsection{Contributions and organization of the paper}
In \autoref{sec_2}, we present an overview of variational inference and the details of CAVI for BPCA, which will be useful to introduce our exponential convergence results for this numerical scheme. Recently, theoretical results about convergence speed of CAVI have been established in several contexts. In \cite{zhang2020theoretical}, the context is community detection, whereas in \cite{plummer2020dynamics} and \cite{goplerud2024partially} the contexts are 2D Ising models and generalized linear mixed models, respectively. More general results are presented under a log-concave framework in \cite{LHZG} and \cite{arnese2024convergence}, and under a general abstract framework in \cite{bhattacharya2023convergence}. The results in \cite{LHZG} and \cite{arnese2024convergence} cannot be exploited here as BPCA does not fall under the log-concave framework. For one of our theoretical results, we leverage recent tools provided in \cite{bhattacharya2023convergence}. By doing so, we highlight the strengths and limitations of such abstract tools.

\textit{We now present our main contributions.}
\begin{itemize}
\item In \autoref{sec_3}, we prove a precise exponential convergence result in the case $k = 1$, meaning that the model includes a single PC. To achieve this, we establish a novel connection between CAVI and the classical \emph{power iteration algorithm} \citep[chapter 25]{trefethen2022numerical} coming from the field of numerical linear algebra whose purpose is to compute the leading singular vectors of a matrix (\autoref{convergence_of_mu_z_fixed_update}).
\item In \autoref{genk_subseq}, we provide the \emph{first} general (for any $k$) theoretical result on the exponential convergence of CAVI for BPCA (\autoref{main_theorem_bpca}). This result being more general than that of \autoref{sec_3}, it requires different tools and is of a slightly different flavor.
\item A key ingredient of the proof of \autoref{main_theorem_bpca} is a \emph{novel lower bound on the symmetric Kullback--Leibler (KL) divergence} between two multivariate normal distributions (\autoref{KL_lowerbound}). We believe this result is of independent interest in information theory as traditional bounds based on Pinsker’s inequality are often too loose to be useful in practice.
\end{itemize}
\section{Variational inference and CAVI}\label{sec_2}
In variational inference, we resort to find a ``\textit{good enough}'' approximation to  the posterior distribution, that we denote by $\pi$, hoping to capture its key statistical properties. We begin by identifying a family $\mathcal{Q}$ of parametric probability distributions for the approximation. Assume, as will be the case in our context of BPCA, that $\pi$ and all distributions in $\mathcal{Q}$ are absolutely continuous with respect to the Lebesgue measure. We then pick the best candidate from the family $\mathcal{Q}$, in the sense of minimizing the KL divergence with respect to the posterior distribution:
\begin{equation}\label{problem}
q^{*} = \argmin\limits_{q \in \mathcal{Q}} \text{KL}(q \| \pi), \quad \text{KL}(q \| \pi) = \int q(\mathrm{d}x) \log\left( \frac{q(x)}{\pi(x)} \right),
\end{equation}
where we used the same notation $q$ and $\pi$ to denote the densities to simplify. 

The optimization problem in \eqref{problem} is typically intractable. We now present CAVI, a numerical scheme that is used to find an approximate solution to this optimization problem. We present CAVI for our specific purpose of variational inference of BPCA. For this purpose, we specify a mean-field variational family defined as follows:
\begin{equation}
  \mathcal{Q} = \mathcal{Q}_{\textbf{W}} \times \mathcal{Q}_{\textbf{Z}} := \left\{ q\ \middle\vert \begin{array}{l}
    q = q_{\textbf{W}} q_{\textbf{Z}} \quad \text{with} \\
    q_{\textbf{W}} = \mathcal{N}( \mu_{\textbf{W}}, \Sigma^{(c)}_{\textbf{W}} \otimes \Sigma^{(r)}_{\textbf{W}})\quad\text{and} \quad
    q_{\textbf{Z}} = \mathcal{N}( \mu_{\textbf{Z}}, \Sigma^{(c)}_{\textbf{Z}} \otimes \Sigma^{(r)}_{\textbf{Z}})
  \end{array} \right\}.
  \label{family}
\end{equation}

In the case of $q_\textbf{W}$, for instance, $\mu_\textbf{W}$ is a $d \times k$ matrix and the covariance parameter $\Sigma^{(c)}_{\textbf{W}}\otimes \Sigma^{(r)}_{\textbf{W}}$ means that the rows of \textbf{W} share the covariance matrix $\Sigma_\textbf{W}^{(r)}$, whereas the columns share the covariance matrix $\Sigma_\textbf{W}^{(c)}$.

At each step of CAVI, a particular marginal distribution, say $q_{\textbf{W}}$ (referred to as a \textit{block}), is updated in a way that decreases the KL divergence, while the other marginal distribution $q_{\textbf{Z}}$ is fixed. For implementing such a scheme, we need to decide the order in which the blocks are updated. In our case, we first proceed with $q_\textbf{W}$, and next $q_\textbf{Z}$. We thus need to initialize $q_\textbf{Z}=q_{\textbf{Z}}^{(0)}$ for the first update of $q_\textbf{W}$. Therefore, CAVI proceeds sequentially as
\begin{align}
\nonumber
    q^{(t+1)}_{\textbf{W}} &= \arg\min_{q_{\textbf{W}}\in\mathcal{Q}_{\textbf{W}}} \, \text{KL}\left(q_{\textbf{W}} q^{(t)}_{\textbf{Z}} \,\|\, \pi\right), \\
    q^{(t+1)}_{\textbf{Z}} &= \arg\min_{q_{\textbf{Z}}\in\mathcal{Q}_{\textbf{Z}}} \, \text{KL}\left(q^{(t+1)}_{\textbf{W}} q_{\textbf{Z}} \,\|\, \pi\right),\quad t=0,1,2,\ldots.\label{problem_star}
\end{align}
 
It can be shown that each sub-problem described in \eqref{problem_star} is convex and has a unique solution:
\begin{equation}\label{iterw}
q_{\textbf{W}}^{(t+1)}(\textbf{W})\propto \exp\int q^{(t)}_{\textbf{Z}}(\textbf{Z})\log p(\textbf{W},\textbf{Z},\textbf{X}) \ \d\textbf{Z},
\end{equation}
\begin{equation}\label{iterz}
q^{(t+1)}_{\textbf{Z}}(\textbf{Z})\propto \exp\int q^{(t+1)}_{\textbf{W}}(\textbf{W})\log p(\textbf{W},\textbf{Z},\textbf{X}) \ \d\textbf{W},
\end{equation}
where $p(\textbf{W},\textbf{Z},\textbf{X})$ corresponds to the joint density ($p$ will be used as a generic notation for density) and \textbf{X} is the $n \times d$ observed data matrix with lines given by $x_1, \ldots, x_n$ (here and afterwards, $x_1, \ldots, x_n$ are used to denote the observed data vectors; we made an abuse of notation to simplify). 

For BPCA, the integrals in \eqref{iterw} and \eqref{iterz} admit closed-form solutions, see \cite{819772}. However, no derivation was provided in that paper. For completeness, we offer a detailed derivation in \autoref{ap3}. We have that $q^{(t+1)}_\textbf{W}$ corresponds to a matrix normal distribution with parameters $\mu_\textbf{W}^{(t+1)}$ and $I_d\otimes \Sigma_\textbf{W}^{(t+1)}$, where
$$\Sigma^{(t+1)}_{\textbf{W}}=(\tau_0(n\Sigma^{(t)}_{\textbf{Z}}+(\mu^{(t)}_{\textbf{Z}})'\mu^{(t)}_{\textbf{Z}})+\Lambda)^{-1}\text{ and }\mu^{(t+1)}_{\textbf{W}}=(\tau_0\textbf{X}'\mu_{\textbf{Z}}^{(t)})\Sigma^{(t+1)}_{\textbf{W}},$$
$\Sigma_\textbf{Z}^{(t)}$ corresponding to the row covariance matrix of $\textbf{Z}$ at iteration $t$. Indeed, we have that $q^{(t+1)}_\textbf{Z}$ corresponds to a matrix normal distribution with parameters $\mu_\textbf{Z}^{(t+1)}$ and $I_n\otimes \Sigma_\textbf{Z}^{(t+1)}$ where
$$\Sigma^{(t+1)}_{\textbf{Z}}=(\tau_0(d\Sigma^{(t+1)}_{\textbf{W}}+(\mu^{(t+1)}_{\textbf{W}})'\mu^{(t+1)}_{\textbf{W}})+I_k)^{-1}\text{ and }\mu^{(t+1)}_{\textbf{Z}}=(\tau_0\textbf{X}'\mu_{\textbf{W}}^{(t+1)})\Sigma^{(t+1)}_{\textbf{Z}}.$$
Note that even if we did not impose an independence between the rows of \textbf{W}, it automatically arises in the updates. This is a consequence of normality assumptions in the BPCA model. For analogous reasons, it is also the case that the rows of \textbf{Z} are independent in the updates. All this suggests to set $q_\textbf{Z}^{(0)}$ to a matrix normal distribution with a covariance parameter of $I_n \otimes\Sigma_Z^{(0)}$ in the initialization step. From now on, we thus assume that CAVI is initialized with such a $q_\textbf{Z}^{(0)}$. Consequently, we can consider without loss of generality that all $(q_\textbf{W}, q_\textbf{Z}) \in \mathcal{Q}$ are such that the covariance parameters of these matrix normal distributions are $I_d\otimes\Sigma_\textbf{W}$ and $I_n\otimes\Sigma_\textbf{Z}$. Note also that if CAVI is initialized with $\mu_\textbf{Z}^{(0)} = 0$,  all next iterates will be such that both $\mu_\textbf{W}^{(t)}$ and $\mu_\textbf{Z}^{(t)}$ equal to $0$ for all $t$. We consider this fixed point of CAVI as trivial and will thus assumed in the following that $\mu_\textbf{Z}^{(0)}\neq 0$.

 CAVI scheme proceeds by repeating those sequential updates until convergence (see \autoref{algos} for a detailed algorithm). To assess the convergence of CAVI scheme, it would be natural to monitor the KL divergence. However, the latter cannot be used directly because it includes the intractable normalizing constant of the posterior distribution (recall \eqref{problem}). We monitor instead what is referred to as the \textit{evidence lower bound} (ELBO) derived from
\begin{equation}\label{elbo_def}
 \text{KL}(q|| \pi)=\log p(\textbf{X})-\text{ELBO}(q),
\end{equation}
where $\text{ELBO}(q)=\mathbb{E}_q[\log p(\textbf{W},\textbf{Z},\textbf{X})]-\mathbb{E}_q[\log q(\textbf{W},\textbf{Z})],$ $\mathbb{E}_q$ denoting an expectation with respect to $q$. It is referred to as the evidence lower bound given that the non-negativity of the KL divergence implies that it is a lower bound on the \textit{log-evidence}, that is $\log p(\textbf{X})$, corresponding to the log of the normalizing constant. For many statistical models involving latent variables, including BPCA, the calculation of $\text{ELBO}(q)$ is tractable. Therefore in practice, it is determined that CAVI convergence is reached when the increase in $\text{ELBO}(q)$ is negligible. 

We finish this section with a theoretical result which guarantees existence of at least one solution to the optimization problem in \eqref{problem} for BPCA.
\begin{Proposition}\label{non_void}
Let $\pi$ be the posterior distribution for BPCA, as described in \autoref{section1.1}. Let $M=\inf_{q\in\mathcal{Q}}\text{KL}\left(q\| \pi\right)$, then there exists $q^*\in\mathcal{Q}$ such that 
$$\text{KL}\left(q^*\| \pi\right)=M.$$
\end{Proposition}
Note that there can be multiple global minima, meaning that different choices of the tuple $(q^{*}_{\textbf{W}}, q^{*}_{\textbf{Z}})$ may result in product densities $q^{*}_{\textbf{W}} q^{*}_{\textbf{Z}}$ that are equally distant from the posterior in terms of KL divergence.
\section{Theoretical analysis of CAVI for $k = 1$}\label{sec_3}\label{1_subseq}
In BPCA, as mentioned, the case $k = 1$ reduces the model to one with a single PC. In particular, $k=1$ implies that the family $\mathcal{Q}$ (recall \eqref{family}) is such that $\mu_{\textbf{W}}\in\mathbb{R}^d$ and $\mu_{\textbf{Z}}\in\mathbb{R}^n$ are vectors, and $\Sigma_\textbf{W}$ and $\Sigma_{\textbf{Z}}$ are positive real numbers. Also, CAVI takes a simpler form governed by the following mappings:
\begin{align*}
  F\colon \mathbb{R}^n\times \mathbb{R}_{+} &\to \mathbb{R}^d\times\mathbb{R}_{+}\\
  (\mu_\textbf{Z},\Sigma_\textbf{Z}) &\mapsto \left(\dfrac{\tau_0 \textbf{X}' \mu_\textbf{Z}}{\tau_0 (n\Sigma_\textbf{Z} + \|\mu_\textbf{Z}\|^2) + \Lambda}, \dfrac{1}{\tau_0 (n\Sigma_\textbf{Z} + \|\mu_\textbf{Z}\|^2) + \Lambda}\right),
\end{align*}
\begin{align*}
  G\colon \mathbb{R}^d\times \mathbb{R}_{+} &\to \mathbb{R}^n\times \mathbb{R}_{+}\\
  (\mu_\textbf{W},\Sigma_\textbf{W}) &\mapsto \left(\dfrac{\tau_0 \textbf{X}\mu_\textbf{W}}{\tau_0 (d\Sigma_\textbf{W} +\|\mu_\textbf{W}\|^2) + 1}, \dfrac{1}{\tau_0 (d\Sigma_{\textbf{W}} + \|\mu_\textbf{W}\|^2) + 1}\right),
\end{align*}
where $\| \cdot \|$ is the Euclidean norm; when applied to a matrix, $\| \cdot \|$ represents the Frobenius norm, unless otherwise stated. CAVI indeed corresponds to the following update equations for the parameters of the matrix normal distributions (which are in fact multivariate normal distributions in the case $k = 1$):
\begin{equation}\label{update_equations}
    (\mu^{(t+1)}_{\textbf{W}},\Sigma^{(t+1)}_{\textbf{W}})=F(\mu^{(t)}_{\textbf{Z}},\Sigma^{(t)}_{\textbf{Z}})\quad\text{and}\quad 
    (\mu^{(t+1)}_{\textbf{Z}},\Sigma^{(t+1)}_{\textbf{Z}})= G(\mu^{(t+1)}_{\textbf{W}},\Sigma^{(t+1)}_{\textbf{W}}),\quad t=0,1,\ldots.
\end{equation}
Thus,
\begin{align}\label{bal}
 (\mu^{(t+1)}_{\textbf{Z}},\Sigma^{(t+1)}_{\textbf{Z}}) &= G \circ F(\mu^{(t)}_{\textbf{Z}},\Sigma^{(t)}_{\textbf{Z}}),\quad t=0,1,\ldots,\ \text{and} \notag \\
 (\mu^{(t+1)}_{\textbf{W}},\Sigma^{(t+1)}_{\textbf{W}}) 
 &= F \circ G(\mu^{(t)}_{\textbf{W}},\Sigma^{(t)}_{\textbf{W}}),\quad t=1,2,\ldots.
\end{align}

To analyze the convergence of CAVI, we leverage these representations and the \textit{polar decomposition} of $\mu_\textbf{Z}^{(t)}$ and $\mu_\textbf{W}^{(t)}$: \[
\mu^{(t)}_{\textbf{Z}} = 
\underbrace{\|\mu^{(t)}_{\textbf{Z}}\|}_{\text{scaling}} 
\underbrace{\left(\frac{\mu^{(t)}_{\textbf{Z}}}{\|\mu^{(t)}_{\textbf{Z}}\|}\right)}_{\text{direction}}\quad\text{and}\quad \mu^{(t)}_{\textbf{W}} = 
\underbrace{\|\mu^{(t)}_{\textbf{W}}\|}_{\text{scaling}} 
\underbrace{\left(\frac{\mu^{(t)}_{\textbf{W}}}{\|\mu^{(t)}_{\textbf{W}}\|}\right)}_{\text{direction}}.
\]
In the following two subsections, the convergence question is addressed in two parts: first for the directional component, and next for the scaling component, together with that of $\Sigma^{(t)}_\textbf{Z}$ and $\Sigma^{(t)}_\textbf{W}$.
\subsection{On the convergence of the directional component}
From now on, we assume that the data matrix $\textbf{X} \in \mathbb{R}^{n \times d}$ has full rank. Using a singular value decomposition of \textbf{X}, we can conclude that $\textbf{X}'\textbf{X}$ and $\textbf{X}\textbf{X}'$ share the same set of non-zero eigenvalues. Therefore, $\textbf{X}'\textbf{X}$ and $\textbf{X}\textbf{X}'$ both have exactly $d$ positive eigenvalues, denoted by $\lambda_1, \ldots, \lambda_d$, and the remaining eigenvalues of $\textbf{X}\textbf{X}'$ are equal to $0$. We also assume that $\lambda_1 > \lambda_2 > ... > \lambda_d$. Let $\mu_1, \mu_2, \ldots, \mu_n$ denote the eigenvectors of $\textbf{X}\textbf{X}'$, forming an orthonormal basis of $\mathbb{R}^n$. Finally, let $\text{sgn}(x)$ denote the sign of a non-zero real number $x$. We are ready to present the main result of this section.
\begin{Theorem}\label{convergence_of_mu_z_fixed_update}
Given an initialization $(\mu^{(0)}_{\textbf{Z}},\Sigma^{(0)}_{\textbf{Z}})\in \mathbb{R}^n\setminus\{0\}\times \mathbb{R}_{+}$, let $c_j= \mu_j'\mu^{(0)}_{\textbf{Z}},$ $j=1,\ldots,n$. Let $i= \min\{1\leqslant j\leqslant n: c_j\neq 0\}$. If $i<d$, then there exist constants $c_0,\ c_0'>0$ which depend on the initialization $\mu^{(0)}_{\textbf{Z}}$ such that for all $t\geqslant 2$, we have 
 $$\left\|\dfrac{\mu^{(t)}_{\textbf{Z}}}{\|\mu^{(t)}_{\textbf{Z}}\|}-\text{sgn}(c_i)\ \mu_{i}\right\|\leqslant c_0\left(1-\dfrac{\rho_i}{\lambda_i}\right)^t,$$
 $$\left\|\dfrac{\mu^{(t)}_{\textbf{W}}}{\|\mu^{(t)}_{\textbf{W}}\|}-\text{sgn}(c_i)\dfrac{{\textbf{X}'\mu_i}}{{\|\textbf{X}'\mu_i\|}}\right\|\leqslant c_0'\left(1-\dfrac{\rho_i}{\lambda_i}\right)^t,$$
 where $\rho_i =\lambda_i-\lambda_{i+1}$.
\end{Theorem}
The proof of \autoref{convergence_of_mu_z_fixed_update} can be found in \autoref{ap_proof_main_theorem}. It is to be noted that ${\textbf{X}'\mu_1}/{\|\textbf{X}'\mu_1\|}$ corresponds to the first principal direction of the data, that is, the direction along which the data matrix \textbf{X} has maximum variance. If $c_1 \neq 0$, then according to \autoref{convergence_of_mu_z_fixed_update} the iterates $\mu_\textbf{Z}^{(t)}/\|\mu_\textbf{Z}^{(t)}\|$ converge exponentially to the first (normalized) PC obtained with usual PCA and the convergence rate is determined by the ratio $\rho_1 / \lambda_1$. We have $c_1 = 0$ if $\mu_\textbf{Z}^{(0)}$ is equal to one of the vectors $\mu_2, ..., \mu_n$ or a linear combination of these. It means that if $\mu_\textbf{Z}^{(0)}$ is set to a realization of a continuous random variable, we are sure (with probability $1$) that this event will not happen and thus that $c_1 \neq 0$. In this case, iterates $\mu^{(t)}_\textbf{W}/\|\mu^{(t)}_\textbf{W}\|$ converge exponentially to the first principal direction. It is natural to obtain such a convergence, as in the case $k = 1$, the orthogonality underlying PCA is automatically obtained in BPCA. When $k > 1$, unfortunately, this is not the case, and the parameters obtained by variational inference of BPCA do not respect orthogonality (as we observed in numerical experiments). In \autoref{fig:convergence_direction}, we show that the bounds in \autoref{convergence_of_mu_z_fixed_update} are tight.
\begin{figure}[ht]
    \centering
    \includegraphics[width=0.8\textwidth]{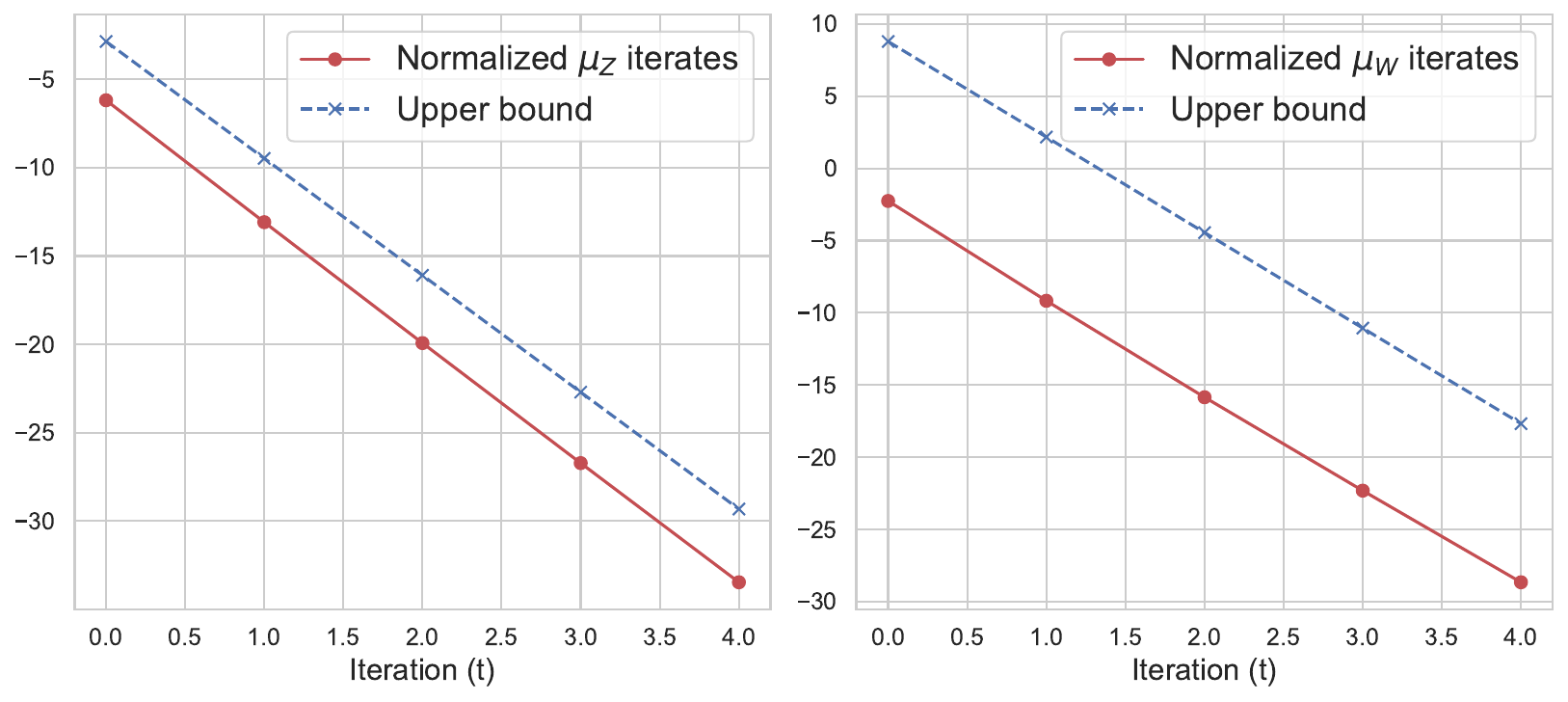}
    \caption{$\|\mu^{(t)}_{\textbf{Z}}/\|\mu^{(t)}_{\textbf{Z}}\|-\text{sgn}(c_1)\mu_1\|$ (left panel) and  $\|\mu^{(t)}_{\textbf{W}}/\|\mu^{(t)}_{\textbf{W}}\|-\text{sgn}(c_1)\textbf{X}'\mu_1/\|\textbf{X}'\mu_1\|\|$ (right panel) as the algorithm progresses, and upper bounds provided in \autoref{convergence_of_mu_z_fixed_update}, on the log scale.}
    \label{fig:convergence_direction}
\end{figure}

In the proof of \autoref{convergence_of_mu_z_fixed_update},  we use that the iterates $\mu_\textbf{Z}^{(t)}/\|\mu_\textbf{Z}^{(t)}\|$ can be written as 
\begin{equation}\label{power_iter}
    \mu_\textbf{Z}^{(t)}/\|\mu_\textbf{Z}^{(t)}\|= \left(\textbf{X}\textbf{X}'\right)^t\mu^{(0)}_\textbf{Z}/ \|\left(\textbf{X}\textbf{X}'\right)^t\mu^{(0)}_\textbf{Z}\|.
\end{equation}
We have an analogous expression for the iterates $\mu_\textbf{W}^{(t)}/\|\mu_\textbf{W}^{(t)}\|$. The expression in \eqref{power_iter} corresponds to the power iteration algorithm that can be used for computing the leading normalized singular vectors of $\textbf{X}\textbf{X}'$. The expression for the iterates $\mu_\textbf{W}^{(t)}/\|\mu_\textbf{W}^{(t)}\|$ corresponds to the power iteration algorithm, but for computing the leading normalized singular vectors of $\textbf{X}'\textbf{X}$.
\subsection{On the convergence of the scaling component and $(\Sigma^{(t)}_{\textbf{Z}},\Sigma^{(t)}_{\textbf{W}})$}
We begin by analyzing the sequence $(\|\mu^{(t)}_{\textbf{Z}}\|,\Sigma^{(t)}_{\textbf{Z}})$. Given \autoref{convergence_of_mu_z_fixed_update}, one can exploit that the normalized version of $\mu_\textbf{Z}^{(t)}$ converges to $\mu_1$ (if $c_1\neq 0$) and initialize CAVI as follows:  
\[
(\mu^{(0)}_{\textbf{Z}},\Sigma^{(0)}_{\textbf{Z}}) = (a^{(0)}\mu_1, b^{(0)}),
\]
where $(a^{(t)}, b^{(t)})\in\mathbb{R}_+^2$ with $a^{(t)}$ representing the norm of $\mu_\textbf{Z}^{(t)}$. We now investigate whether the sequence $(a^{(t)}, b^{(t)})$ admits a limit.

With the initialization $(\mu^{(0)}_{\textbf{Z}},\Sigma^{(0)}_{\textbf{Z}}) = (a^{(0)}\mu_1, b^{(0)}),$ we are in fact guaranteed that $c_1\neq 0$ in \autoref{convergence_of_mu_z_fixed_update} and CAVI updates are governed by the following mapping:  
\begin{align*}
  \Phi\colon \mathbb{R}_{+}^2 &\to \mathbb{R}_{+}^2\\
  (a,b) &\mapsto \left(\dfrac{\tau_0^2 a L(a,b)\lambda_1}{d\tau_0L(a,b) +\tau_0^3a^2\lambda_1+L(a,b)^2}, \dfrac{L(a,b)^2}{d\tau_0L(a,b) +\tau_0^3a^2\lambda_1+L(a,b)^2}\right),
\end{align*}
where $L(a,b) = \tau_0(nb + a^2) + \Lambda$. In other words,  
\[
(a^{(t+1)},b^{(t+1)}) = \Phi(a^{(t)},b^{(t)}), \quad t=0,1,\ldots.
\]
If the sequence $(a^{(t)}, b^{(t)})$ converges to a limit $(a^*, b^*)$, then this limit must satisfy the fixed point condition $\Phi(a^*, b^*) = (a^*, b^*)$ given that the map $\Phi$ is continuous. Consequently, characterizing all such fixed points provides insights into the possible limit points of the CAVI iterates. The next proposition is about this.
\begin{Proposition}
\label{fixed_pts_possibility}
Assume that $(a^*,b^*)$ is a fixed point of $\Phi$. Then, $(a^*)^2$ must be a root of the following quadratic polynomial:
$$P(u)=\lambda_1\tau_0^2 u^2+\tau_0 u \Bigl[2\lambda_1\Lambda+(d-\lambda_1\tau_0)(\lambda_1\tau_0-n)+(\lambda_1\tau_0-n)^2\Bigr]
+[\lambda_1\Lambda^2+(d-\lambda_1\tau_0)(\lambda_1\tau_0-n)\Lambda].$$
Given $(a^*)^2$ a positive root of $P$, if any, then $b^*$ is uniquely given by
$$b^* = \dfrac{\Lambda+\tau_0\left(a^*\right)^2}{\tau_0^2\lambda_1-n\tau_0}.$$
Consequently, the map $\Phi$ has at most $2$ fixed points.
\end{Proposition}
The proof of \autoref{fixed_pts_possibility} can be found in \autoref{ap_proof_main_theorem}. A user can exploit \autoref{fixed_pts_possibility} to conclude whether $(a^{(t)}, b^{(t)})$, and thus CAVI, converge. Indeed, as noted in \cite{arnese2024convergence}, CAVI iterates may not always converge. The user can numerically solve the equation $P(u) = 0$. If there is no positive root, then CAVI does not converge. If there is at least one, then the user can apply the map $\Phi$ on $(a^*, b^*)$ given in \autoref{fixed_pts_possibility}, and verify that the output is equal to $(a^*, b^*)$. If this is the case,  then $(a^*, b^*)$ is a fixed point and CAVI converges to this fixed point. In this case, CAVI actually does not need to be implemented as the parameters of the normal distribution of \textbf{Z} can be set to $(a^* \mu_1, b^* I_n)$ and that of \textbf{W} can be set to $(\mu_\textbf{W}^*,  \Sigma_\textbf{W}^* I_d)$ with $(\mu_\textbf{W}^*, \Sigma_\textbf{W}^*)$ computed using $F$  (see \eqref{update_equations}). If there are two positive roots, the user should also compute the ELBO values to find the optimal solution. The polynomial $P$ has a unique positive root if the constant term is strictly negative, which is the case if  
\((\lambda_1)^2 \tau_0^2 - \lambda_1 (\Lambda + (d + n) \tau_0) + d n > 0\).  
This holds when \(\lambda_1\) lies outside the interval \((\alpha, \beta)\), where  
\(\alpha = \frac{\Lambda + (d + n) \tau_0 - \sqrt{(\Lambda + (d + n) \tau_0)^2 - 4 \tau_0^2 d n}}{2 \tau_0^2}\) and   
\(\beta = \frac{\Lambda + (d + n) \tau_0 + \sqrt{(\Lambda + (d + n) \tau_0)^2 - 4 \tau_0^2 d n}}{2 \tau_0^2}\). We observed all this in numerical experiments (see \autoref{numerical_k1} for an example).

A result from dynamical systems \citep[Theorem 10.1]{teschl2012ordinary} states that, if the Jacobian of $\Phi$ evaluated at $(a^*, b^*)$ has eigenvalues strictly smaller than $1$ in absolute value, then the convergence of $(a^{(t)}, b^{(t)})$ is exponential. Applying this result in our context is not feasible as the map $\Phi$ is complex and the solution of $P$ implicit. Therefore, it is not possible using such tools to provide a theoretical guarantee of exponential convergence of $(a^{(t)}, b^{(t)})$ and thus of CAVI iterates (recall \autoref{convergence_of_mu_z_fixed_update}). However, a user can continue the procedure mentioned above for finding the optimal solution $(a^*, b^*)$ (if it exists) and numerically evaluate the Jacobian of $\Phi$ to verify the exponential convergence. We observed in numerical experiments that the Jacobian has eigenvalues strictly lesser than $1$ in absolute value and exponential convergence; see \autoref{fig:ab_scaling_convergence} for an example. 

Even though we partially address the question of exponential convergence of CAVI using dynamical systems tools, we believe that the results of this section in the case $k = 1$  are of interest as they allow a precise understanding of CAVI and variational inference in this case (for instance, that the same results as PCA are retrieved). In the next section, we provide a general result for any $k$ about exponential convergence of CAVI.

\begin{figure}[h]
    \centering
    \includegraphics[width=0.9\textwidth]{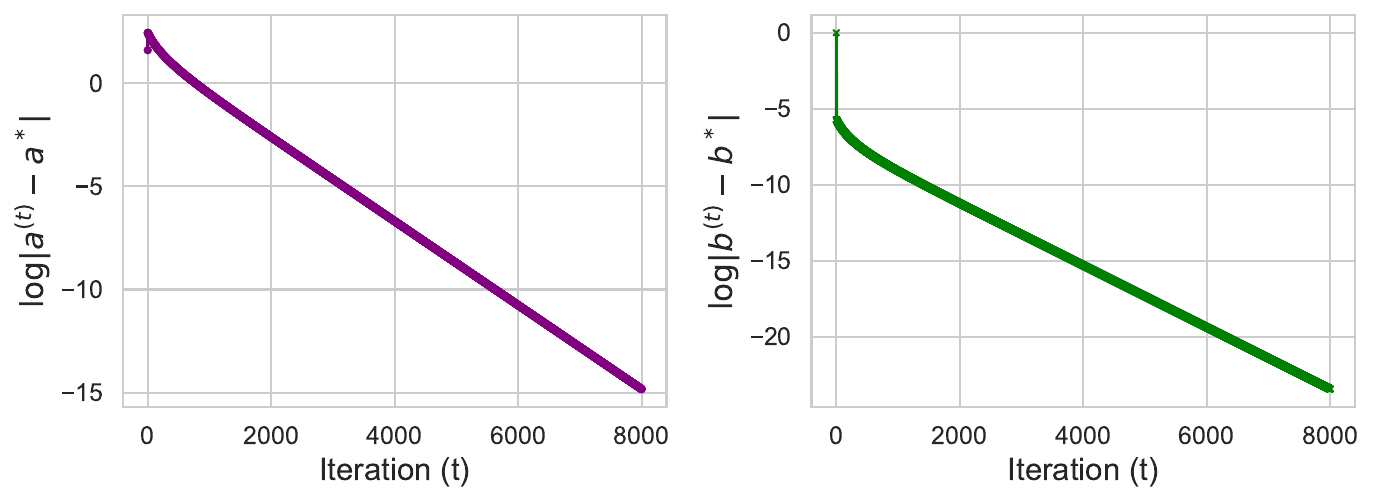}
    \caption{$|a^{(t)} - a^*|$ (left panel) and $|b^{(t)} - b^*|$ (right panel) as the algorithm progresses, on log scale.}
    \label{fig:ab_scaling_convergence}
\end{figure}
\section{Theoretical analysis of CAVI in the general case}\label{genk_subseq}
In this section, we analyze CAVI for BPCA when $k$ is an arbitrary integer. We view $\text{KL}(q\| \pi)$ as a function of the variational parameters of $q$. Thus, we identify a tuple $q= (\mathcal{N}(\mu_{\textbf{W}},I_d\otimes \Sigma_{\textbf{W}}), \mathcal{N}(\mu_{\textbf{Z}},I_n\otimes \Sigma_{\textbf{Z}}))$ by its corresponding parameters $(\mu_{\textbf{W}},\Sigma_{\textbf{W}},\mu_{\textbf{Z}},\Sigma_{\textbf{Z}})\in \mathbb{R}^{d\times k}\times \mathbb{S}^k_{+}\times\mathbb{R}^{n\times k}\times\mathbb{S}^k_{+}$, where $\mathbb{S}^k_{+}$ denotes the set of all positive definite symmetric matrices, and we will thus abuse notation by writing $q=(\mu_{\textbf{W}},\Sigma_{\textbf{W}},\mu_{\textbf{Z}},\Sigma_{\textbf{Z}})$ to simplify. We define the map $\Psi:\mathbb{R}^{d\times k}\times \mathbb{S}^k_{+}\times\mathbb{R}^{n\times k}\times\mathbb{S}^k_{+}\to\mathbb{R}$ as
\begin{equation}\label{loss_func_first_appear}
\Psi \left(\mu_{\textbf{W}},\Sigma_{\textbf{W}},\mu_{\textbf{Z}},\Sigma_{\textbf{Z}}\right)= \text{KL}\left(q_{\textbf{W}}q_{\textbf{Z}}\| \pi\right) -\log p(\textbf{X}).
\end{equation}
The optimization problem in \eqref{problem} is equivalent to minimizing $\Psi$. 

Let $\mathcal{M}$ denote the set of all stationary points (which includes minimizers) to the optimization problem in \eqref{problem}. In \autoref{non_void}, we proved that $\mathcal{M}\neq \varnothing$ and therefore we are allowed to consider a stationary point \( q^* = (\mu^*_{\textbf{W}},\Sigma^*_{\textbf{W}},\mu^*_{\textbf{Z}},\Sigma^*_{\textbf{Z}})\in\mathcal{M} \). For any $(q_{\textbf{W}}, q_{\textbf{Z}})$, it is shown \citep[Lemma 2]{bhattacharya2023convergence} that
\begin{align}
\text{KL}\left(q_{\textbf{W}} q_{\textbf{Z}} \,\|\, \pi\right) 
- \text{KL}\left(q^*_{\textbf{W}} q^*_{\textbf{Z}} \,\|\, \pi\right)
= \text{KL}(q_{\textbf{W}} \,\|\, q^*_{\textbf{W}})
+ \text{KL}(q_{\textbf{Z}} \,\|\, q^*_{\textbf{Z}})
- \Delta_{q^*}(q_{\textbf{W}}, q_{\textbf{Z}}),
\label{interaction}
\end{align}
where
\begin{equation}\label{covariance}
\Delta_{q^*}(q_\textbf{W},q_{\textbf{Z}})= \iint \left(q_\textbf{W}(\textbf{W})-q^*_{\textbf{W}}(\textbf{W})\right)\left(q_\textbf{Z}(\textbf{Z})-q^*_{\textbf{Z}}(\textbf{Z})\right)\log \pi(\textbf{W,\ \textbf{Z}}) \ \d\textbf{W} \ \d\textbf{Z}.
\end{equation}
The mathematical object in \eqref{covariance} was first introduced in \cite{bhattacharya2023convergence}. It controls the CAVI dynamics: it acts as a measure of covariance between the blocks $q_\textbf{W}$ and $q_\textbf{Z}$. It is worth noting that if the posterior density would factorize exactly into two marginal densities, that is, if $\textbf{W}$ and $\textbf{Z}$ would be \textit{a posteriori} independent, then $\Delta_{q^*}(q_{\textbf{W}}, q_{\textbf{Z}})$ would vanish for all $q_{\textbf{W}}$ and $q_{\textbf{Z}}$. In this case, the optimization problem in \eqref{problem} is equivalent to separate minimizations of $ \text{KL}(q_{\textbf{W}} \,\|\, q^*_{\textbf{W}})$ and $ \text{KL}(q_{\textbf{Z}} \,\|\, q^*_{\textbf{Z}})$, leading to a convergence of CAVI in one step using equations \eqref{iterw} and \eqref{iterz}.

We thus expect a fast convergence of CAVI when $\Delta_{q^*}(q_\textbf{W},q_\textbf{Z})$ is small, relatively to $\text{KL}(q_{\textbf{W}}\|q^*_{\textbf{W}})$ and $\text{KL}(q_{\textbf{Z}}\|q^*_{\textbf{Z}})$, on the trajectory. In \cite{bhattacharya2023convergence}, a sufficient condition, related to what is referred to as the $1/2$-\textit{generalized correlation}, is proposed:
\begin{equation}\label{gen_correlation_first_appear}
  \text{GCorr}_{1/2}(r_0):=\sup\limits_{\substack{q_{\textbf{W}}\in \mathcal{B}^*_{\textbf{W}}(r_0)\setminus \{q^*_\textbf{W}\},\\q_{\textbf{Z}}\in \mathcal{B}^*_{\textbf{Z}}(r_0)\setminus \{q^*_\textbf{Z}\}}}
\dfrac{|\Delta_{q^*}(q_{\textbf{W}},q_{\textbf{Z}})|}
{\sqrt{D_{\text{KL},1/2}\left(q_{\textbf{W}}\|q^*_{\textbf{W}}\right)
D_{\text{KL},1/2}\left(q_{\textbf{Z}}\|q^*_{\textbf{Z}}\right)}}\in (0,2),  
\end{equation}
where $\mathcal{B}^*_{\textbf{Z}}(r_0)=\{q_{\textbf{Z}}:\text{KL}(q^*_{\textbf{Z}}\|q_{\textbf{Z}})\leqslant r_0\}$ is a KL ball of radius $r_0$ around $q_{\textbf{Z}}^*$, with an analogous definition for $\mathcal{B}^*_{\textbf{W}}(r_0)$, and $D_{\text{KL},1/2}(q_\textbf{Z}\|q^*_{\textbf{Z}})=(\text{KL}(q_\textbf{Z}\|q^*_{\textbf{Z}})+\text{KL}(q^*_\textbf{Z}\|q_{\textbf{Z}}))/2$ is the symmetric KL divergence between $q_\textbf{Z}$ and $q_\textbf{Z}^*$, with again an analogous definition for $D_{\text{KL},1/2}(q_\textbf{W}\|q^*_{\textbf{W}})$. In Theorem 4 of \cite{bhattacharya2023convergence}, it is assumed that CAVI is initialized within a symmetric KL ball, that is $q^{(0)}_{\textbf{Z}}\in\{q_{\textbf{Z}}:D_{\text{KL},1/2}(q_{\textbf{Z}}\|q^*_{\textbf{Z}})\leqslant r_0/2\},$ with $\text{GCorr}(r_0)\in(0,2)$, thus effectively ensuring that the generalized correlation is small on the trajectory. The denominator in the generalized correlation is there to make sure that $\Delta_{q^*}(q_\textbf{W},q_{\textbf{Z}})$ is small relatively to $D_{\text{KL},1/2}(q_{\textbf{Z}}\|q^*_{\textbf{Z}})$ and $D_{\text{KL},1/2}(q_{\textbf{W}}\|q^*_{\textbf{W}}),$ which is a stronger requirement than being small relatively to $\text{KL}(q_{\textbf{Z}}\|q^*_{\textbf{Z}})$ and $\text{KL}(q_{\textbf{W}}\|q^*_{\textbf{W}})$.

In our context of BPCA, by using the expression of $\log \pi(\textbf{W}, \ \textbf{Z})$ in \eqref{covariance} and after some simplifications, the condition $\text{GCorr}_{1/2}(r_0)\in(0,2)$ corresponds to 
\begin{equation}\label{GCorr_rewritten}
 \sup\limits_{\substack{q_{\textbf{W}}\in \mathcal{B}^*_{\textbf{W}}(r_0)\setminus \{q^*_\textbf{W}\},\\q_{\textbf{Z}}\in \mathcal{B}^*_{\textbf{Z}}(r_0)\setminus \{q^*_\textbf{Z}\}}}
\dfrac{\left|\displaystyle\iint \left(q_{\textbf{W}}(\textbf{W}) - q^*_{\textbf{W}}(\textbf{W})\right)\left(q_{\textbf{Z}}(\textbf{Z}) - q^*_{\textbf{Z}}(\textbf{Z})\right) \|\textbf{X} - \textbf{Z}\textbf{W}'\|^2 \, \mathrm{d}\textbf{W} \, \mathrm{d}\textbf{Z}\right|}
{\sqrt{D_{\text{KL},1/2}\left(q_{\textbf{W}}\|q^*_{\textbf{W}}\right)
D_{\text{KL},1/2}\left(q_{\textbf{Z}}\|q^*_{\textbf{Z}}\right)}}\in(0,2).  
\end{equation}
To have a small $\Delta_{q^*}(q_\textbf{W},q_{\textbf{Z}})$ relatively to $D_{\text{KL},1/2}(q_{\textbf{Z}}\|q^*_{\textbf{Z}})$ and $D_{\text{KL},1/2}(q_{\textbf{W}}\|q^*_{\textbf{W}})$, we thus need to have small values for $\|\textbf{X} - \textbf{Z}\textbf{W}'\|$ over the integral. In other words, \textbf{X} must be sufficiently well-approximated by the low-rank factorization $\textbf{Z}\textbf{W}'$ over the integral.

To verify the condition, or equivalently to have a uniform control of the generalized correlation on the trajectory, we developed a sharp lower bound on the symmetric KL between normal distributions (\autoref{KL_lowerbound}). The result is of independent interest in information theory, particularly in the context of sparse source retrieval (see \cite{ghodhbani2019close}). This bound yields a product of terms like $\|\mu_\textbf{Z} - \mu_\textbf{Z}^*\|\|\mu_\textbf{W} - \mu_\textbf{W}^*\|$. For the numerator, using the properties of the trace operator (see the proof of \autoref{delta_upperbound}), we can write
$$\Delta_{q^*}\left(q_\textbf{W},q_\textbf{Z}\right)= \tau_0 \text{tr}\left(\left(\mu_{\textbf{W}}-\mu^*_{\textbf{W}}\right)\left(\mu'_{\textbf{Z}}-[\mu^*_{\textbf{Z}}]'\right)\textbf{X}\right)-\dfrac{\tau_0}{2}\text{tr}\left(\left(\Gamma_{\textbf{W}}-\Gamma_{\textbf{W}}^*\right)\left(\Gamma'_{\textbf{Z}}-[\Gamma^*_{\textbf{Z}}]'\right)\right),$$
where $\Gamma_{\textbf{W}}=d\Sigma_{\textbf{W}}+\mu_{\textbf{W}}'\mu_{\textbf{W}}$ and $\Gamma_{\textbf{Z}}=n\Sigma_{\textbf{Z}}+\mu_{\textbf{Z}}'\mu_{\textbf{Z}}$, $\Gamma^*_{\textbf{W}}$ and $ \Gamma^*_{\textbf{Z}}$ being defined analogously. We then performed a careful analysis of the numerator to obtain an upper bound involving terms that cancel out with aforementioned terms in the denominator. This yields a bound on the ratio in \eqref{GCorr_rewritten} which is independent of $q_\textbf{Z}$ and $q_\textbf{W}$, guaranteeing a uniform control on the generalized correlation on the trajectory. The uniform bound obtained translates into: the parameters of the tuple $(q_\textbf{Z}^*, q_\textbf{W}^*)$ is required to belong to $\mathcal{I}\subset \mathbb{R}^{d\times k}\times \mathbb{S}^k_{+}\times\mathbb{R}^{n\times k}\times\mathbb{S}^k_{+}$, denoting the set of parameters satisfying the condition 
\begin{equation}\label{GC}
\max \Bigg\{  
\tau_0\left(\|\Sigma_{\mathbf{W}}\| \cdot \|\Sigma_{\mathbf{Z}}\|\right)^{1/2}  
\frac{\|\mathbf{X}\| + 2\|\mu_{\mathbf{W}}\| \cdot \|\mu_{\mathbf{Z}}\|}{\gamma_0},\ 
e\tau_0 \left(\|\Sigma_{\mathbf{W}}\| \cdot \|\Sigma_{\mathbf{Z}}\|\right) \sqrt{nd},\ 
\end{equation}
\begin{equation*}
\tau_0\left(\|\Sigma_{\mathbf{W}}\| \cdot \|\Sigma_{\mathbf{Z}}\|^{1/2}\right)  
\frac{\sqrt{2de} \|\mu_{\mathbf{Z}}\|}{\gamma_0^{1/2}},\ 
\tau_0\left(\|\Sigma_{\mathbf{W}}\|^{1/2} \cdot \|\Sigma_{\mathbf{Z}}\|\right)  
\frac{\sqrt{2ne} \|\mu_{\mathbf{W}}\|}{\gamma_0^{1/2}}  
\Bigg\} < 1,
\end{equation*}
where $\gamma_0=(1+e^{-1})/4$. We now state two assumptions that will be used in the main result of this section.
\begin{Assumption}\label{A1}
The intersection $\mathcal{M}\cap\mathcal{I}$ is non-empty, that is, at least one stationary point $q^*$ has parameters satisfying the condition in \eqref{GC}.
\end{Assumption}
\begin{Assumption}\label{A2}
The loss function in \eqref{loss_func_first_appear} has a non-singular Hessian at every stationary point.
\end{Assumption}
\begin{Theorem}[Local contraction of KL]\label{main_theorem_bpca}
Suppose Assumptions \ref{A1} and \ref{A2} hold. Let $q^*\in\mathcal{M}\cap\mathcal{I}$. There exists $r^*>0$ such that, if the initialization in \autoref{algo1} (see \autoref{algos}) satisfies $q^{(0)}_\textbf{Z}\in\mathcal{B}^*_{\textbf{Z}}(r^*)$, then for any $t\geqslant 1$,
$${\text{KL}}(q^{(t+1)}_\textbf{Z}\| q^*_{\textbf{Z}})\leqslant C_{\textbf{Z}}\kappa^t,$$
$${\text{KL}}(q^{(t+1)}_\textbf{W}\| q^*_{\textbf{W}})\leqslant C_{\textbf{W}}\kappa^t,$$
where $C_\textbf{Z}$ and $C_\textbf{W}$ are two positive constants, independent of $t$, and $\kappa\in (0, 1)$ is a constant.
\end{Theorem}
The proof of \autoref{main_theorem_bpca} can be found in \autoref{gen_k_details}. Empirically, we observed that the critical quantity in \eqref{GC} is first one, namely
$$\tau_0\left(\|\Sigma_{\mathbf{W}}\| \cdot \|\Sigma_{\mathbf{Z}}\|\right)^{1/2}  
\frac{\|\mathbf{X}\| + 2\|\mu_{\mathbf{W}}\| \cdot \|\mu_{\mathbf{Z}}\|}{\gamma_0},$$
in the sense that the three others in \eqref{GC} were observed to be smaller. The first quantity reflects the extent to which \(\mathbf{X}\) is well-approximated by a low-rank factorization (note that it is the only one depending explicitly on \textbf{X}). In some numerical experiments, we were able to verify that this quantity was smaller than \(1\).

 \autoref{A2} is a regularity condition that ensures that two stationary points $q_1^*, q_2^* \in\mathcal{M}$ cannot be arbitrarily close \cite[Corollary 2.3]{milnor1963morse}. This assumption is implicit in \cite{bhattacharya2023convergence}. In our case, it is possible to express \(\Psi\) explicitly (see \eqref{eq:loss_function}). However, theoretically proving the non-singularity of its Hessian at stationary points is infeasible using standard tools due to its complexity, and thus beyond the scope of this paper. As shown in \autoref{balls_dont_intersect} (see \autoref{ap_localbehaviour}), \autoref{A2} is incompatible with choosing the prior hyperparameter $\Lambda$ proportional to the identity matrix. \autoref{balls_dont_intersect} states that such a $\Lambda$ leads to $\mathcal{M}$ containing a continuous manifold with an infinite number of stationary points with the same distance in KL divergence to the posterior distribution. This equivalence between the stationary points is due to the non-identifiability of the model mentioned in \autoref{section1.1}. Numerical experiments suggest that selecting $\Lambda$ to be non-proportional to the identity matrix is sufficient to satisfy \autoref{A2} (see \autoref{ap_localbehaviour}).

 \autoref{main_theorem_bpca} has the following corollary regarding the convergence of the parameters of the matrix normal distributions $q^{(t)}_\textbf{Z}=\mathcal{N}(\mu^{(t)}_\textbf{Z}, I_n\otimes \Sigma^{(t)}_\textbf{Z})$ and $q^{(t)}_\textbf{W}=\mathcal{N}(\mu^{(t)}_\textbf{W}, I_d\otimes \Sigma^{(t)}_\textbf{W})$.
\begin{Corollary}[Convergence of parameters]\label{parameter}
Under the same assumptions as \autoref{main_theorem_bpca}, one has, for any $t\geqslant 1$,
$$\max\{\|\mu^{(t+1)}_{\textbf{Z}}-\mu^*_{\textbf{Z}}\|,\ \|\Sigma^{(t+1)}_{\textbf{Z}}-\Sigma^*_{\textbf{Z}}\|\}\leqslant A_{\textbf{Z}}\kappa^{t/2},$$
$$\max\{\|\mu^{(t+1)}_{\textbf{W}}-\mu^*_{\textbf{W}}\|,\ \|\Sigma^{(t+1)}_{\textbf{W}}-\Sigma^*_{\textbf{W}}\|\}\leqslant A_{\textbf{W}}\kappa^{t/2},$$
where $A_{\textbf{Z}},A_{\textbf{W}},$ are positive constants, independent of $t$. The constant $\kappa$ is the same as in \autoref{main_theorem_bpca}.
\end{Corollary}
The proof of \autoref{parameter} can be found in \autoref{gen_k_details}.

With \autoref{parameter}, we do not have a characterization of the limit point $q^*$ and of the rate of convergence. In particular, we do not have a characterization of $\mu_\textbf{Z}^*$ and $\mu_\textbf{W}^*$. As mentioned, we empirically observed that these parameters do not correspond to traditional PCA when $k > 1$, by opposition to the case $k = 1$. This is expected as the orthogonality underlying PCA is not imposed in BPCA. A strength of the dynamical systems tools used to derive the results in \autoref{sec_3} is to allow for more precision in these results, but they cannot be applied when $k > 1$. A strength of the abstract tools of \cite{bhattacharya2023convergence} is that they do. This section allows one to familiarize oneself with these tools.
\section{Limitations and broader impact}
A limitation of our work lies in the assumptions of \autoref{main_theorem_bpca}. In particular, we expect the result to hold under a more general version of the condition in \eqref{GC}, with an upper bound perhaps depending on the problem parameters and the data. The condition in \eqref{GC} follows from the sufficient condition in \cite{bhattacharya2023convergence} (see \eqref{gen_correlation_first_appear}), highlighting at the same time a limitation of their abstract tools. We also expect \autoref{A2} to not be necessary. Proving a general result like \autoref{main_theorem_bpca} in \cite{bhattacharya2023convergence} without this assumption would require a significantly different proof technique. Instead of analyzing the local behavior of CAVI with a uniform control of the generalized correlation within a KL ball, one would need to directly work with CAVI trajectories. In other words, our work highlights a need of theory in the line of research of speed of convergence CAVI to cover a broader range of situations.

A limitation of BPCA is that the number of PCs used in the model, that is $k$, is considered as an hyperparameter, thus as known and fixed by the user. In practice, it is of interest to learn this (hyper)parameter to identify the number of PCs to include in a data-driven way. A possible extension of our work is to establish exponential convergence of CAVI in the case where $k$ is learned using the procedure of \cite{cherief2019consistency}.

It is worth noting that our research does not present any discernible negative societal implications.
\bibliographystyle{apalike}
\bibliography{ref}

\appendix
\section{An explicit CAVI algorithm for BPCA}\label{algos}
In this section, we present an explicit CAVI algorithm in detail.
\begin{algorithm}[htp]
\caption{CAVI algorithm for BPCA} \label{algo1}

 \begin{enumerate}
 \itemsep 1.9mm

  \item Initialization: Fix $(\mu^{(0)}_{\textbf{Z}},\Sigma^{(0)}_{\textbf{Z}})\in\mathbb{R}^{n\times k}\times\mathbb{S}_k^+$, and $q_\textbf{Z}^{(0)}=\mathcal{N}(\mu^{(0)}_{\textbf{Z}},I_n\otimes \Sigma^{(0)}_{\textbf{Z}})$. Fix $\varepsilon>0$.
  \item \text{Update the parameters of the distribution of \textbf{W}}:
      $$\Sigma^{(t+1)}_{\textbf{W}}=(\tau_0(n\Sigma^{(t)}_{\textbf{Z}}+(\mu^{(t)}_{\textbf{Z}})'\mu^{(t)}_{\textbf{Z}})+\Lambda)^{-1},$$
      $$\mu^{(t+1)}_{\textbf{W}}=(\tau_0\textbf{X}'\mu_{\textbf{Z}}^{(t)})\Sigma^{(t+1)}_{\textbf{W}}.$$
  \item \text{Update the parameters of the distribution of the latent variable \textbf{Z}}:
     $$\Sigma^{(t+1)}_{\textbf{Z}}=(\tau_0(d\Sigma^{(t+1)}_{\textbf{W}}+(\mu^{(t+1)}_{\textbf{W}})'\mu^{(t+1)}_{\textbf{W}})+I_k)^{-1},$$
     $$\mu^{(t+1)}_{\textbf{Z}}=(\tau_0\textbf{X}\mu_{\textbf{W}}^{(t+1)})\Sigma^{(t+1)}_{\textbf{Z}}.$$
  \item\text{Calculate the relative amount of increase in ELBO:}
  $$\Delta^{(t)}\text{ELBO}:=\dfrac{\text{ELBO}(q^{(t+1)})-\text{ELBO}(q^{(t)})}{\text{ELBO}(q^{(t)})}.$$
  \item While $\Delta^{(t)}\text{ELBO}>\varepsilon$, go to Step 2.

 \end{enumerate}

\end{algorithm}
\section{Details of the analysis of CAVI dynamics when $k=1$}\label{ap_proof_main_theorem}
\begin{proof}[Proof of \autoref{convergence_of_mu_z_fixed_update}]
Recall that the $n$ eigenvalues of the matris $\textbf{X}\textbf{X}'$ are sorted such that the first $d$ eigen values are strictly positive $\lambda_1>\lambda_2>\lambda_3> \ldots>\lambda_d>0$ and $ \lambda_{d+1}=\cdots=\lambda_n =0$ with corresponding eigenvectors $\mu_1,\mu_2,\ldots, \mu_n$, which forms an orthonormal basis of $\mathbb{R}^n$. Without loss of generality, we assume there exist constants $c_1,c_2,c_3,\ldots, c_n$ such that $\tilde{\mu}^{(0)}_{\textbf{Z}}:=\mu^{(0)}_{\textbf{Z}}/\|\mu^{(0)}_{\textbf{Z}}\|=c_1\mu_1+c_2\mu_1+\cdots+c_n\mu_n$ with $c_1\neq 0$.

For $t\geqslant 1$, let us define the quantity $B_t= \tau_0(n\Sigma^{(t)}_{\textbf{Z}}+\|\mu^{(t)}_{\textbf{Z}}\|^2)+\Lambda$. From \eqref{update_equations}, we have that
\begin{align}
\nonumber
(\mu^{(t+1)}_{\textbf{Z}},\Sigma^{(t+1)}_{\textbf{Z}})
&= (G\circ F)(\mu^{(t)}_{\textbf{Z}},\Sigma^{(t)}_{\textbf{Z}})\\ 
\nonumber
&= G\left(\dfrac{\tau_0\textbf{X}'\mu^{\text{(t)}}_{\textbf{Z}}}{\tau_0(n\Sigma^{\text{(t)}}_{\textbf{Z}}+\|\mu^{(t)}_{\textbf{Z}}\|^2)+\Lambda}\ , \dfrac{1}{\tau_0(n\Sigma^{\text{(t)}}_{\textbf{Z}}+\|\mu^{(t)}_{\textbf{Z}}\|^2)+\Lambda}\ \right).
\end{align}
Thus, we have that
\begin{align}
\nonumber
\mu^{(t+1)}_{\textbf{Z}}
&=\dfrac{\tau_0\textbf{X} \left(\dfrac{\tau_0\textbf{X}'\mu^{\text{(t)}}_{\textbf{Z}}}{\tau_0(n\Sigma^{\text{(t)}}_{\textbf{Z}}+\|\mu^{(t)}_{\textbf{Z}}\|^2)+\Lambda}\right)}{\tau_0\left(\dfrac{d}{\tau_0(n\Sigma^{\text{(t)}}_{\textbf{Z}}+\|\mu^{(t)}_{\textbf{Z}}\|^2)+\Lambda}+\left\|\dfrac{\tau_0\textbf{X}'\mu^{\text{(t)}}_{\textbf{Z}}}{\tau_0(n\Sigma^{\text{(t)}}_{\textbf{Z}}+\|\mu^{(t)}_{\textbf{Z}}\|^2)+\Lambda}\right\|^2\right)+1}\\
\nonumber
&= \dfrac{\tau_0^2\textbf{X}\textbf{X}'\mu^{(t)}_{\textbf{Z}}/B_t}{\tau_0\left(\dfrac{d}{B_t}+\dfrac{\|\tau_0\textbf{X}'\mu^{(t)}_{\textbf{Z}}\|^2}{B_t^2}\right)+1}\\
&= \dfrac{\tau_0^2B_t\textbf{X}\textbf{X}'\mu^{(t)}_{\textbf{Z}}}{d\tau_0B_t+\tau_0^3\|\textbf{X}'\mu^{(t)}_{\textbf{Z}}\|^2+B_t^2} \label{mu_z_update},
\end{align}
and 
$$\Sigma^{(t+1)}_{\textbf{Z}}=\dfrac{B_t^2}{d\tau_0B_t+\tau_0^3\|\textbf{X}'\mu^{(t)}_{\textbf{Z}}\|^2+B_t^2},$$
which implies the normalized $(t+1)$-th update of $\tilde{\mu}^{(t+1)}_{\textbf{Z}}= \textbf{X}\textbf{X}'\mu^{(t)}_{\textbf{Z}}/\|\textbf{X}\textbf{X}'\mu^{(t)}_{\textbf{Z}}\|=\textbf{X}\textbf{X}'\tilde{\mu}^{(t)}_{\textbf{Z}}/\|\textbf{X}\textbf{X}'\tilde{\mu}^{(t)}_{\textbf{Z}}\|$. This motivates us to look at the following iterative sequence on the unit circle:
$$\tilde{\mu}^{(0)}_{\textbf{Z}}:=\dfrac{{\mu}^{(0)}_{\textbf{Z}}}{\|{\mu}^{(0)}_{\textbf{Z}}\|}\ ,\dfrac{\textbf{A}\tilde{\mu}^{(0)}_{\textbf{Z}}}{\|\textbf{A}\tilde{\mu}^{(0)}_{\textbf{Z}}\|}\ ,\dfrac{\textbf{A}^2\tilde{\mu}^{(0)}_{\textbf{Z}}}{\|\textbf{A}^2\tilde{\mu}^{(0)}_{\textbf{Z}}\|}\ , \ldots,\quad\text{where }\textbf{A}=\textbf{X}\textbf{X}'.$$
We claim the above sequence converges to $\mu_1$. Note that for any $t\geqslant 1$,
\begin{align*}
\textbf{A}^t \tilde{\mu}^{(0)}_\textbf{Z} 
&= c_1\textbf{A}^t\mu_1+c_2\textbf{A}^t\mu_1+\cdots+c_n\textbf{A}^t\mu_n \\
&= c_1(\lambda_1)^t\Biggl(\mu_1
+\dfrac{c_2}{c_1}\left(\dfrac{\lambda_2}{\lambda_1}\right)^t\mu_1+\cdots+\dfrac{c_{d}}{c_1}\left(\dfrac{\lambda_d}{\lambda_1}\right)^t\mu_d\Biggr).
\end{align*}

Let us define the vector
$$v_t=\left(\mu_1+\dfrac{c_2}{c_1}\left(\dfrac{\lambda_2}{\lambda_1}\right)^t\mu_1+\cdots+\dfrac{c_{d}}{c_1}\left(\dfrac{\lambda_d}{\lambda_1}\right)^t\mu_d\right).$$
Thus we have that
\begin{align}
\nonumber
\tilde{\mu}^{(t)}_{\textbf{Z}}
&= \dfrac{\textbf{A}^t\tilde{\mu}^{(0)}_{\textbf{Z}}}{\|\textbf{A}^t\tilde{\mu}^{(0)}_{\textbf{Z}}\|}\ \\
\nonumber 
&=\dfrac{c_1}{|c_1|}\left(\dfrac{v_t}{\|v_t\|}\right)\\ 
\nonumber
&= \dfrac{c_1}{|c_1|}\dfrac{\mu_1+\dfrac{c_2}{c_1}\left(\dfrac{\lambda_2}{\lambda_1}\right)^t\mu_1+\cdots+\dfrac{c_{d-1}}{c_1}\left(\dfrac{\lambda_d}{\lambda_1}\right)^t\mu_d}{\left\|\mu_1+\dfrac{c_2}{c_1}\left(\dfrac{\lambda_2}{\lambda_1}\right)^t\mu_1+\cdots+\dfrac{c_{d-1}}{c_1}\left(\dfrac{\lambda_d}{\lambda_1}\right)^t\mu_d\right\|}\\
\nonumber
&\to \text{sgn}(c_1)\mu_1, \quad \text{as } t\to\infty.
\end{align}
The second part of the proof follows from \autoref{rate_lemma} after noticing for all $t\geqslant 1$
$$\|v_t-\text{sgn}(c_1)\mu_1\|\leqslant   \max\{|c_1|,|c_2|,\ldots,|c_{d}|\}\left(\dfrac{d}{|c_1|}\right)\left(\dfrac{\lambda_2}{\lambda_1}\right)^t=\dfrac{c_0}{2}\left(1-\dfrac{\rho_1}{\lambda_1}\right)^t,$$
for some $c_0>0$ which only depends on $d$, the initialization and the data $\textbf{X}$, and consequently 
$$\left\|\tilde{\mu}^{(t)}_{\textbf{Z}}-\dfrac{c_1}{|c_1|}\mu_1\right\|=\left\|\dfrac{c_1}{|c_1|}\left(\dfrac{v_t}{\|v_t\|}\right)-\dfrac{c_1}{|c_1|}\mu_1\right\|=\left\|\dfrac{v_t}{\|v_t\|}-\mu_1\right\|\leqslant c_0\left(1-\dfrac{\rho_1}{\lambda_1}\right)^t.$$
To obtain the upper bound on normalized iterates of $\mu^{(t)}_{\textbf{W}}$, we pretend that our initialization is 
$$\mu^{(1)}_{\textbf{W}}=\dfrac{\tau_0\textbf{X}'\mu^{(0)}_{\textbf{Z}}}{B_0}=\dfrac{\tau_0}{B_0}\left(c_1\textbf{X}'\mu_1+\cdots +c_n\textbf{X}'\mu_n\right)=\dfrac{\tau_0}{B_0}\left(\sqrt{\lambda_1}c_1\nu_1+\cdots+\sqrt{\lambda_d}c_d\nu_d\right),$$
where $\nu_1,\ldots,\nu_d$ are unit norm eigen vectors of $\textbf{X}'\textbf{X}$ corresponding to the eigen values $\lambda_1,\ldots,\lambda_d$, forming an orthonormal basis of $\mathbb{R}^d$. Using \eqref{bal}, a similar argument as before leads to: 
$$\dfrac{\mu^{(t+1)}_\textbf{W}}{\|\mu^{(t+1)}_\textbf{W}\|}= \dfrac{\textbf{X}'\textbf{X}\mu^{(t)}_\textbf{W}}{\|\textbf{X}'\textbf{X}\mu^{(t)}_\textbf{W}\|}\quad t=1,2,\ldots.$$
In this case, using an analogous calculation as $\mu^{(t)}_\textbf{Z}$, we obtain for $t\geqslant 2,$
$$\left\|\dfrac{\mu^{(t+1)}_\textbf{W}}{\|\mu^{(t+1)}_\textbf{W}\|}-\text{sgn}(c_1)\dfrac{\textbf{X}'\mu_1}{\|\textbf{X}'\mu_1\|}\right\|\leqslant c_0'\left(1-\dfrac{\rho_1}{\lambda_1}\right)^t,$$
where $c_0'= \tau_0c_0\sqrt{\lambda_1}/\Lambda$.
\end{proof} 
\begin{proof}[Proof of \autoref{fixed_pts_possibility}]
We recall that
\begin{align*}
  \Phi\colon \mathbb{R}_{+}^2 &\to \mathbb{R}_{+}^2\\
  (a,b) &\mapsto \left(\dfrac{\tau_0^2 a L(a,b)\lambda_1}{d\tau_0L(a,b) +\tau_0^3a^2\lambda_1+L(a,b)^2}, \dfrac{L(a,b)^2}{d\tau_0L(a,b) +\tau_0^3a^2\lambda_1+L(a,b)^2}\right),
\end{align*}
where $L(a,b)=\tau_0(nb+a^2)+\Lambda$. Suppose, $(a^*,b^*)$ is a fixed point of $\Psi$ with $a^*>0$, i.e., we have
$$(a^*,b^*)=\Phi(a^*,b^*),$$
equivalently, the following system of equations is satisfied
$$a^* = \dfrac{\tau_0^2 a^* L(a^*,b^*)\lambda_1}{d\tau_0L(a^*,b^*) +\tau_0^3(a^*)^2\lambda_1+L(a^*,b^*)^2}\text{ and }b^* = \dfrac{L(a^*,b^*)^2}{d\tau_0L(a^*,b^*) +\tau_0^3(a^*)^2\lambda_1+L(a^*,b^*)^2}.$$
Plugging the value of $b^*$ from the second equation, we have that
\begin{equation}\label{b_star}
 a^* = \dfrac{\tau_0^2L(a^*,b^*)a^*\lambda_1b^*}{L(a^*,b^*)^2}\implies b^*= \dfrac{L(a^*,b^*)}{\tau^2_0\lambda_1}.  
\end{equation}
Equating two expressions for $b^*$, we obtain
$$\dfrac{L(a^*,b^*)}{\tau^2_0 \lambda_1}=\dfrac{L(a^*,b^*)^2}{d\tau_0L(a^*,b^*) +\tau_0^3(a^*)^2\lambda_1+L(a^*,b^*)^2}.$$
Thus,
\begin{equation}\label{eq_L}
    L(a^*,b^*)^2+(d\tau_0-\lambda_1\tau_0^2)L(a^*,b^*)+\tau_0^3(a^*)^2\lambda_1 = 0. 
\end{equation}
Note that also, $b^* =\dfrac{1}{n}\left(\dfrac{L(a^*,b^*)-\Lambda}{\tau_0}-(a^*)^2\right)$ which implies,

$$\dfrac{L(a^*,b^*)}{\tau^2_0 \lambda_1}=\dfrac{1}{n}\left(\dfrac{L(a^*,b^*)-\Lambda}{\tau_0}-(a^*)^2\right)\implies L(a^*,b^*)=\dfrac{\frac{\Lambda}{n\tau_0}+\frac{(a^*)^2}{n}}{\frac{1}{n\tau_0}-\frac{1}{\tau_0^2\lambda_1}}.$$
 Note that if we plug the value of $L(a^*,b^*)$ in \eqref{eq_L}, then we can obtain a quadratic equation, which is being satisfied by $(a^*)^2$. Let $u=(a^*)^2$. Then 
$$L(a^*,b^*)= \tau_0\lambda_1\left(\dfrac{\Lambda+\tau_0u}{\lambda_1\tau_0-n}\right),$$
substituting this in \eqref{eq_L}, one obtains a quadratic equation in $u$, which means $(a^*)^2$ has at most $2$ possibilities implying the conclusion. The details are as follows:
\begin{align*}
\left(\tau_0\lambda_1\frac{\Lambda+\tau_0 u}{\lambda_1\tau_0-n}\right)^2 
&+(d\tau_0-\lambda_1\tau_0^2)\left(\tau_0\lambda_1\frac{\Lambda+\tau_0 u}{\lambda_1\tau_0-n}\right)
+\tau_0^3\lambda_1 u &= 0, \\
\tau_0^2(\lambda_1)^2(\Lambda+\tau_0 u)^2 
&+(d\tau_0-\lambda_1\tau_0^2)\tau_0\lambda_1(\Lambda+\tau_0 u)(\lambda_1\tau_0-n)
+\tau_0^3\lambda_1 u (\lambda_1\tau_0-n)^2 &= 0, \\
\tau_0^2(\lambda_1)^2(\Lambda+\tau_0 u)^2 
&+\tau_0^2\lambda_1(d-\lambda_1\tau_0)(\Lambda+\tau_0 u)(\lambda_1\tau_0-n)
+\tau_0^3\lambda_1 u (\lambda_1\tau_0-n)^2 &= 0, \\
\lambda_1(\Lambda+\tau_0 u)^2 
&+(d-\lambda_1\tau_0)(\Lambda+\tau_0 u)(\lambda_1\tau_0-n)
+\tau_0 u (\lambda_1\tau_0-n)^2 &= 0, \\
\lambda_1(\Lambda^2+2\Lambda\tau_0 u+\tau_0^2 u^2) 
&+(d-\lambda_1\tau_0)(\lambda_1\tau_0-n)\Lambda \\
&\quad+(d-\lambda_1\tau_0)(\lambda_1\tau_0-n)\tau_0 u+\tau_0 u (\lambda_1\tau_0-n)^2 &= 0, \\
\lambda_1\tau_0^2 u^2 
&+\tau_0 u \Bigl[2\lambda_1\Lambda+(d-\lambda_1\tau_0)(\lambda_1\tau_0-n)+(\lambda_1\tau_0-n)^2\Bigr] \\
&\quad+[\lambda_1\Lambda^2+(d-\lambda_1\tau_0)(\lambda_1\tau_0-n)\Lambda] &= 0.
\end{align*}
Finally from \eqref{b_star} and the definition of $L(a^*,b^*)$, 
$$b^* = \dfrac{L(a^*,b^*)}{\tau_0^2
\lambda_1}=\dfrac{\tau_0(nb^*+(a^*)^2)+\Lambda}{\tau_0^2
\lambda_1},$$
and solving for $b^*$, we obtain
$$b^* = \dfrac{\Lambda+\tau_0(a^*)^2}{\tau_0^2\lambda_1-n\tau_0}.$$
\end{proof}
\section{Details of the analysis of CAVI dynamics for $k\in\mathbb{N}$}\label{gen_k_details}
In this appendix, we prove a result stronger than \autoref{main_theorem_bpca}. We provide an upper bound on the symmetric KL for the CAVI iterates which implies an upper bound on their KL divergence to the target distributions.
\begin{proof}[Proof of \autoref{main_theorem_bpca}]
Let $r_0>0$ be fixed, and we pick some $(q_{\textbf{W}},q_{\textbf{Z}})\in \mathcal{B}^*_{\textbf{W}}(r_0)\times \mathcal{B}^*_{\textbf{Z}}(r_0)$. Let us denote the quantities: $A_1=\|\mu_{\textbf{Z}}-\mu^*_{\textbf{Z}}\|, A_2=\|\mu_{\textbf{W}}-\mu^*_{\textbf{W}}\|,B_1=\|\Sigma_{\textbf{Z}}-\Sigma^*_{\textbf{Z}}\|,B_2=\|\Sigma_{\textbf{W}}-\Sigma^*_{\textbf{W}}\|,C_1=\|\mu_{\textbf{Z}}\|+\|\mu^*_{\textbf{Z}}\|,C_2=\|\mu_{\textbf{W}}\|+\|\mu^*_{\textbf{W}}\|$. Recall the generalized correlation in \eqref{gen_correlation_first_appear} as:
$$
\text{GCorr}_{1/2}(r_0):=\sup\limits_{\substack{q_{\textbf{W}}\in \mathcal{B}^*_{\textbf{W}}(r_0)\setminus \{q^*_\textbf{W}\},\\q_{\textbf{Z}}\in \mathcal{B}^*_{\textbf{Z}}(r_0)\setminus \{q^*_\textbf{Z}\}}}
\dfrac{|\Delta_{q^*}(q_{\textbf{W}},q_{\textbf{Z}})|}
{\sqrt{D_{\text{KL},1/2}\left(q_{\textbf{W}}\|q^*_{\textbf{W}}\right)
D_{\text{KL},1/2}\left(q_{\textbf{Z}}\|q^*_{\textbf{Z}}\right)}}.
$$
Following \cite{bhattacharya2023convergence}, to establish \autoref{main_theorem_bpca}, it suffices to show the existence of $r_0 >0$ such that $\text{GCorr}_{1/2}(r_0)\in (0,2)$. Using \autoref{delta_upperbound} , we have that 
$$\text{GCorr}_{1/2}(r_0)\leqslant \rho_0,$$
where 
$$\rho_0=\dfrac{ \left(\tau_0\|\textbf{X}\|+\dfrac{C_1C_2\tau_0}{2}\right)A_1A_2+(nd\tau_0B_1B_2+dC_1\tau_0 B_2A_1+ nC_2\tau_0 A_2B_1)/2}{\sqrt{Q_1Q_2}}$$
with $Q_1=D_{\text{KL},1/2}\left(q_{\textbf{W}}\|q^*_{\textbf{W}}\right)$ and $Q_2=D_{\text{KL},1/2}\left(q_{\textbf{Z}}\|q^*_{\textbf{Z}}\right)$.
From \autoref{KL_lowerbound} we have that
\begin{align*}
Q_1 &\geqslant \dfrac{dc_3}{\|\Sigma^*_{\textbf{W}}\|^2 e^{1+2(r_0/d)}}\|\Sigma_{\textbf{W}}-\Sigma^*_{\textbf{W}}\|^2+\dfrac{c_2}{\|\Sigma^*_{\textbf{W}}\|}\|\mu_{\textbf{W}}-\mu^*_{\textbf{W}}\|^2\\
&= L_{\textbf{W}^*}^2 B_2^2 +  M_{\textbf{W}^*}^2 A_2^2,
\end{align*}
where
$$L_{\textbf{W}^*}^2 =  \dfrac{dc_3}{\|\Sigma^*_{\textbf{W}}\|^2 e^{1+2(r_0/d)}}\quad{\text{and}}\quad M_{\textbf{W}^*}^2=\dfrac{c_2}{\|\Sigma^*_{\textbf{W}}\|}.$$
Analogously using the same lemmas we obtain 
$$Q_2\geqslant L_{\textbf{Z}^*}^2 B_1^2 +  M_{\textbf{Z}^*}^2 A_1^2,$$
where
$$L_{\textbf{Z}^*}^2 =  \dfrac{nc_3}{\|\Sigma^*_{\textbf{Z}}\|^2 e^{1+2(r_0/n)}}\quad{\text{and}}\quad M_{\textbf{Z}^*}^2=\dfrac{c_2'}{\|\Sigma^*_{\textbf{Z}}\|}.$$
Using Cauchy-Schwarz inequality, we therefore have that
$$\sqrt{Q_1}\geqslant \dfrac{1}{\sqrt{2}}(L_{\textbf{W}^*} B_2+  M_{\textbf{W}^*} A_2)\quad{\text{and}}\quad \sqrt{Q_2}\geqslant \dfrac{1}{\sqrt{2}}(L_{\textbf{Z}^*} B_1 +  M_{\textbf{Z}^*} A_1),$$
Consequently,
$$\sqrt{Q_1Q_2}\geqslant \dfrac{1}{2}\left( M_{\textbf{W}^*}M_{\textbf{Z}^*}A_1A_2+L_{\textbf{W}^*}L_{\textbf{Z}^*}B_1B_2+L_{\textbf{W}^*}M_{\textbf{Z}^*}B_2A_1+M_{\textbf{W}^*}L_{\textbf{Z}^*}A_2B_1\right).$$
Thus we have that
\begin{align*}
\rho_0&=\dfrac{\left(\tau_0\|\textbf{X}\|+\dfrac{C_1C_2\tau_0}{2}\right)A_1A_2+(nd\tau_0B_1B_2+dC_1\tau_0 B_2A_1+ nC_2\tau_0 A_2B_1)/2}{\sqrt{Q_1Q_2}}\\
&\leqslant
2\dfrac{ \left(\tau_0\|\textbf{X}\|+\dfrac{C_1C_2\tau_0}{2}\right)A_1A_2+(nd\tau_0B_1B_2+dC_1\tau_0 B_2A_1+ nC_2\tau_0 A_2B_1)/2}{ M_{\textbf{W}^*}M_{\textbf{Z}^*}A_1A_2+L_{\textbf{W}^*}L_{\textbf{Z}^*}B_1B_2+L_{\textbf{W}^*}M_{\textbf{Z}^*}B_2A_1+M_{\textbf{W}^*}L_{\textbf{Z}^*}A_2B_1}\\
&\leqslant 2 \max\left\{\dfrac{\tau_0\|\textbf{X}\|+\dfrac{\tau_0  C_1C_2}{2}}{ M_{\textbf{W}^*}M_{\textbf{Z}^*}},\dfrac{nd\tau_0}{2 L_{\textbf{W}^*}L_{\textbf{Z}^*}},\dfrac{dC_1\tau_0}{ 2L_{\textbf{W}^*}M_{\textbf{Z}^*}},\dfrac{nC_2\tau_0}{2 M_{\textbf{W}^*}L_{\textbf{Z}^*}}\right\}\quad\text{(by \autoref{algeraic_lemma})},
\end{align*}
where all the quantities $C_1,C_2,L_{\textbf{W}^*},L_{\textbf{Z}^*},M_{\textbf{W}^*},M_{\textbf{Z}^*}$ depend on $r_0$. For the condition as stated in \cite{bhattacharya2023convergence} (see Theorem 3.3) to be satisfied we need to be able to find an $r_0>0$ such that the following 
$$\max\left\{\dfrac{\tau_0\|\textbf{X}\|+\dfrac{\tau_0  C_1C_2}{2}}{ M_{\textbf{W}^*}M_{\textbf{Z}^*}},\dfrac{nd\tau_0}{2 L_{\textbf{W}^*}L_{\textbf{Z}^*}},\dfrac{dC_1\tau_0}{ 2L_{\textbf{W}^*}M_{\textbf{Z}^*}},\dfrac{nC_2\tau_0}{2 M_{\textbf{W}^*}L_{\textbf{Z}^*}}\right\}< 1 $$
holds. Denoting $\gamma_0 = (1+e^{-1})/4$, if $r_0$ is small enough, and in particular if $r_0\to 0$, then we observe that each term in the max remains bounded above by its corresponding counterpart in the condition in \eqref{GC}, 
since $C_1\leqslant 2\|\mu^*_{\textbf{Z}}\|,C_2\leqslant 2\|\mu^*_{\textbf{W}}\|,L^2_{\textbf{Z}^*}\to nc_3/(e\|\Sigma^*_{\textbf{Z}}\|^2),\ L^2_{\textbf{W}^*}\to dc_3/(e\|\Sigma^*_{\textbf{W}}\|^2),\ M^2_{\textbf{W}^*}\to \gamma_0/\|\Sigma^*_{\textbf{W}}\|$ and $M^2_{\textbf{Z}^*}\to \gamma_0/\|\Sigma^*_{\textbf{Z}}\|$ by \autoref{control_lemma} (after plugging back the values $c_2,\ c_2'$ and $c_3$ from \autoref{KL_lowerbound}). Thus, it is sufficient to require that $(\mu^*_{\textbf{W}},\Sigma^*_{\textbf{W}},\mu^*_{\textbf{Z}},\Sigma^*_{\textbf{Z}})\in \mathcal{I}$ since we can then select $r_0>0$ such that $\text{GCorr}_{1/2}(r_0)\in(0,2)$. Finally, we choose $r^*>0$ be sufficiently small such that $q^{(0)}_\textbf{Z}\in\mathcal{B}^*_{\textbf{Z}} (r^*)$ implies $D_\text{KL,1/2}(q^{(0)}_{\textbf{Z}}\|q^*_{\textbf{Z}})\leqslant r_0/2$ (this is possible because of \autoref{neurips_paper_remark}).
\end{proof}
\begin{Remark}\label{neurips_paper_remark}
As shown in \cite{bhattacharya2023convergence}, above proof yields $D_{\text{KL},1/2}(q^{(t)}_{\textbf{W}}\|\ q^*_{\textbf{W}})\leqslant C'_{\textbf{W}}\kappa^t$ and $D_{\text{KL},1/2}(q^{(t)}_{\textbf{Z}}\|\ q^*_{\textbf{Z}})\leqslant C'_{\textbf{Z}}\kappa^t$ for some constants $C'_\textbf{W},\ C'_\textbf{Z}>0$, which is a stronger result than the exponential convergence of KL as stated in \autoref{main_theorem_bpca}. The key reason why we can state \autoref{main_theorem_bpca} entirely in terms of KL divergence is due a recent work \citep{zhang2023properties}. In particular, since \( q^{(0)}_\mathbf{Z} \) and \( q^*_\mathbf{Z} \) are matrix normal distributions with a special covariance structure that imposes row independence, the KL divergence decomposes as a sum over independent rows each of which is a multivariate normal. This tensorization property of KL implies that controlling \( \mathrm{KL}(q^*_\mathbf{Z} \,\|\, q^{(0)}_\mathbf{Z}) \) is sufficient to bound the reverse \( \mathrm{KL}(q^{(0)}_\mathbf{Z} \,\|\, q^*_\mathbf{Z}) \), and hence the symmetrized KL required in the hypothesis of Theorem 4 in \cite{bhattacharya2023convergence}.

\end{Remark}
\begin{proof}[Proof of \autoref{parameter}]
Let $r_0>0$ such that $\text{GCorr}_{1/2}(r_0)\in(0,2)$ which exists as \autoref{A1} and \autoref{A2} are satisifed. Suppose, $r^*$ be sufficiently small such that $q^{(0)}_\textbf{Z}\in\mathcal{B}^*_{\textbf{Z}} (r^*)$ implies $D_\text{KL,1/2}(q^{(0)}_{\textbf{Z}}\|q^*_{\textbf{Z}})\leqslant r_0/2$ (this is possible because of \autoref{neurips_paper_remark}). Following \cite{bhattacharya2023convergence}, it can be shown that $q^{(t)}_\textbf{Z}\in\mathcal{B}^*_{\textbf{Z}} (r_0)$ for all $t\geqslant 1$. Recall \autoref{KL_lowerbound} and set
$$m_{\textbf{Z}}=\min \left\{\dfrac{nc_3}{\|\Sigma^*_{\textbf{Z}}\|^2 e^{1+(2r_0/n)}}, \dfrac{c_2'}{\|\Sigma^*_{\textbf{Z}}\|}\right\}.$$
Using the lower bounds in \autoref{KL_lowerbound} and \autoref{neurips_paper_remark}, we have that
$$m_{\textbf{Z}}\|\Sigma^{(t+1)}_{\textbf{Z}}-\Sigma^*_{\textbf{Z}}\|^2 \leqslant D_{\text{KL},1/2}\left(q^{(t+1)}_{\textbf{Z}}\|\ q^*_{\textbf{Z}}\right)\leqslant C'_{\textbf{Z}}\kappa^t,$$
and 
$$m_{\textbf{Z}}\|\mu^{(t+1)}_{\textbf{Z}}-\mu^*_{\textbf{Z}}\|^2 \leqslant D_{\text{KL},1/2}\left(q^{(t+1)}_{\textbf{Z}}\|\ q^*_{\textbf{Z}}\right)\leqslant C'_{\textbf{Z}}\kappa^t.$$
Therefore,
$$\max\{\|\mu^{(t+1)}_{\textbf{Z}}-\mu^*_{\textbf{Z}}\|,\ \|\Sigma^{(t+1)}_{\textbf{Z}}-\Sigma^*_{\textbf{Z}}\|\}\leqslant A_{\textbf{Z}}\kappa^{t/2},$$
where $A_{\textbf{Z}}=\sqrt{C'_{\textbf{Z}}/m_{\textbf{Z}}}$, a positive constant. The proof of the other inequality is similar. 
\end{proof}
\section{A primer on coercive functions and proof of \autoref{non_void}}
\begin{Definition}
    A function \( f: \mathbb{R}^n \to \mathbb{R} \) is said to be \emph{coercive} if 
    \[
    \|x\| \to \infty \quad \Rightarrow \quad f(x) \to \infty,
    \]
    i.e., \( f(x) \) grows to infinity whenever the norm \( \|x\| \) of the input vector tends to infinity.
\end{Definition}

We now state a standard result without proof:

\begin{Theorem}
    A coercive function on a closed subset \( C \subset \mathbb{R}^n \) attains a global minimum.
\end{Theorem}
The above result will be one of the key ingredients in the following proof. 
\begin{proof}[Proof of \autoref{non_void}]
We recall that if $q_\textbf{W}= \mathcal{N}(\mu_{\textbf{W}},I_d\otimes \Sigma_{\textbf{W}})$ and $q_\textbf{Z}= \mathcal{N}(\mu_{\textbf{Z}},I_n\otimes \Sigma_{\textbf{Z}})$, then the loss function $\Psi$ is such that
$$\Psi \left(\mu_{\textbf{W}},\Sigma_{\textbf{W}},\mu_{\textbf{Z}},\Sigma_{\textbf{Z}}\right)= \text{KL}\left(q_{\textbf{W}}q_{\textbf{Z}}\| \pi\right) -\log p(\textbf{X}).$$
The function $\Psi$ corresponds to $-\text{ELBO}(q)$ (recall \eqref{elbo_def}), where
$$\text{ELBO}(q)=\mathbb{E}_q[\log p(\textbf{W},\textbf{Z},\textbf{X})]-\mathbb{E}_q[\log q]\quad \text{and}\ q= q_{\textbf{W}}q_{\textbf{Z}}.$$
It can be deduced that (see \autoref{elbo_derivation}) 
\begin{align*}
\Psi\left(\mu_{\textbf{W}}, \Sigma_{\textbf{W}}, \mu_{\textbf{Z}}, \Sigma_{\textbf{Z}}\right) 
=& -\tau_0 \text{tr}\left(\mu_{\textbf{W}} \mu_{\textbf{Z}}' \textbf{X}\right)
+ \frac{\tau_0}{2} \text{tr}\left(\Gamma_{\textbf{W}} \Gamma_{\textbf{Z}}\right) + \frac{d}{2} \text{tr}\left(\Lambda \Sigma_{\textbf{W}}\right) 
+ \frac{1}{2} \text{tr}\left(\mu_{\textbf{W}} \Lambda \mu_{\textbf{W}}'\right) \\
& + \frac{1}{2} \text{tr}\left(\Gamma_{\textbf{Z}}\right) 
- \frac{d}{2} \log \det\left(\Sigma_{\textbf{W}}\right) 
- \frac{n}{2} \log \det\left(\Sigma_{\textbf{Z}}\right) + \text{cst},
\end{align*}
where 
$$\Gamma_{\textbf{W}}=d\Sigma_{\textbf{W}}+\mu_{\textbf{W}}'\mu_{\textbf{W}}\text{  and  }\Gamma_{\textbf{Z}}=n\Sigma_{\textbf{Z}}+\mu_{\textbf{Z}}'\mu_{\textbf{Z}}.$$
Plugging in the above values and then completing the square one obtains
\begin{align}
\Psi\left(\mu_{\textbf{W}}, \Sigma_{\textbf{W}}, \mu_{\textbf{Z}}, \Sigma_{\textbf{Z}}\right) &= \frac{\tau_0}{2} \left\|\textbf{X}- \mu_{\textbf{Z}} \mu_{\textbf{W}}' \right\|^2 + \frac{1}{2} \operatorname{tr}\left( \mu_{\textbf{W}}' \Lambda \mu_{\textbf{W}} \right) + \frac{1}{2} \operatorname{tr}\left( \mu_{\textbf{Z}}' \mu_{\textbf{Z}} \right) \notag \\
&\quad + \frac{d}{2} \operatorname{tr}\left( \Lambda\, \Sigma_{\textbf{W}} \right) + \frac{n}{2} \operatorname{tr}\left( \Sigma_{\textbf{Z}} \right) + \frac{\tau_0 d n}{2} \operatorname{tr}\left( \Sigma_{\textbf{W}}\, \Sigma_{\textbf{Z}} \right) \notag \\
&\quad + \frac{\tau_0 d}{2} \operatorname{tr}\left( \Sigma_{\textbf{W}}\, \mu_{\textbf{Z}}' \mu_{\textbf{Z}} \right) + \frac{\tau_0 n}{2} \operatorname{tr}\left( \mu_{\textbf{W}}' \mu_{\textbf{W}}\, \Sigma_{\textbf{Z}} \right) \notag \\
&\quad - \frac{d}{2} \log \det\left( \Sigma_{\textbf{W}} \right) - \frac{n}{2} \log \det\left( \Sigma_{\textbf{Z}} \right)+ \text{cst}. \label{eq:loss_function}
\end{align}

Ignoring constant term if we let
\begin{align*} 
\Psi_0\left(\mu_{\textbf{W}}, \Sigma_{\textbf{W}}, \mu_{\textbf{Z}}, \Sigma_{\textbf{Z}}\right) &= \Psi\left(\mu_{\textbf{W}}, \Sigma_{\textbf{W}}, \mu_{\textbf{Z}}, \Sigma_{\textbf{Z}}\right)-\text{cst}\\
&\geqslant \frac{\tau_0}{2} \left\|\textbf{X}- \mu_{\textbf{Z}} \mu_{\textbf{W}}' \right\|^2 + \frac{1}{2} \lambda_{\text{min}}\left(\Lambda\right) \|\mu_{\textbf{W}}\|^2 
+ \frac{1}{2} \|\mu_{\textbf{Z}}\|^2 + \frac{d}{2} \operatorname{tr}\left( \Lambda\, \Sigma_{\textbf{W}} \right) + \frac{n}{2} \operatorname{tr}\left( \Sigma_{\textbf{Z}} \right) \\
&\quad  + \frac{\tau_0 d n}{2} \operatorname{tr}\left( \Sigma_{\textbf{W}}\, \Sigma_{\textbf{Z}} \right) - \frac{d}{2} \log \det\left( \Sigma_{\textbf{W}} \right) - \frac{n}{2} \log \det\left( \Sigma_{\textbf{Z}} \right),
\end{align*}
then it is enough to show that $\Psi_0$ attains a global minima on the set $\mathbb{R}^{d\times k}\times \mathbb{S}^k_{+}\times\mathbb{R}^{n\times k}\times\mathbb{S}^k_{+}$. We have that
\begin{align}
\nonumber
\Psi_0\left(\mu_{\textbf{W}}, \Sigma_{\textbf{W}}, \mu_{\textbf{Z}}, \Sigma_{\textbf{Z}}\right)
\geqslant & \frac{1}{2} \lambda_{\text{min}}\left(\Lambda\right) \|\mu_{\textbf{W}}\|^2 
+ \frac{1}{2} \|\mu_{\textbf{Z}}\|^2 \\
\nonumber
& + \frac{1}{2} \left(d \lambda_{\text{min}}\left(\Lambda\right)\text{tr}\left(\Sigma_{\textbf{W}}\right)
- d\log \det \left(\Sigma_{\textbf{W}}\right) \right) \\
& + \frac{1}{2} \left( n \text{tr}\left(\Sigma_{\textbf{Z}}\right) 
- n\log \det \left(\Sigma_{\textbf{Z}}\right) \right).\label{psi_0}
\end{align}
Let us define the following closed subset for a fixed parameter $\varepsilon >0$ as
$$\mathcal{M}_{\varepsilon} = \left\{(\mu_{\textbf{W}},\Sigma_{\textbf{W}},\mu_{\textbf{Z}},\Sigma_{\textbf{Z}}): \|\Sigma_{\textbf{W}}\|_{\text{op}}\geqslant \varepsilon \text{ and }\|\Sigma_{\textbf{Z}}\|_{\text{op}}\geqslant \varepsilon\right\}\subset \mathbb{R}^{d\times k}\times \mathbb{S}^k_{+}\times\mathbb{R}^{n\times k}\times\mathbb{S}^k_{+},$$
where $\|\cdot\|_{\text{op}}$ represents the operator norm. It is clear from \eqref{psi_0} that the restriction $\Psi_0 \rvert_{\mathcal{M}_{\varepsilon}}$
 is a continuous and coercive function on the closed subspace $\mathcal{M}_{\varepsilon}$. Therefore, $\Psi_0$ attains a global minima on the set $\mathcal{M}_{\varepsilon}$. We call this minima $m_{\varepsilon}$. Since $\Psi_{0}\to\infty$, if either $\|\Sigma_\textbf{W}\|$ or $\|\Sigma_\textbf{Z}\|$ approaches to $0$, therefore, we can select $\varepsilon^* <\varepsilon $ such that $\Psi_{0}(x) > m_{\varepsilon}$ for all $x\in \mathcal{M}^{\text{c}}_{\varepsilon^*}.$ If we let $m_{\varepsilon^*}$ denote the global minima on the closed subspace $\mathcal{M}_{\varepsilon^*}$, then $m_{\varepsilon^*} \leqslant m_{\varepsilon}$ as $\mathcal{M}_{\varepsilon}\subset\mathcal{M}_{\varepsilon^*}$, and consequently $m_{\varepsilon^*}$ is a global minima of $\Psi_0$ on $\mathbb{R}^{d\times k}\times \mathbb{S}^k_{+}\times\mathbb{R}^{n\times k}\times\mathbb{S}^k_{+}$.
\end{proof}
\section{The Matrix Normal Distribution and its properties}\label{ap1}
We follow the exposition in \cite{gupta1999matrix}. A matrix normal distribution with mean parameter $\mu \in \mathbb{R}^{d \times n}$ and covariance structure $\Sigma^{(c)}\otimes \Sigma^{(r)}$ is defined through the following probability density function (PDF):
\begin{align}
\nonumber
f(\textbf{U};\mu,\Sigma^{(c)}\otimes\Sigma^{(r)})
&=(2\pi)^{-dn/2}|\Sigma^{(c)}|^{-n/2}|\Sigma^{(r)}|^{-d/2}\\
\nonumber
&\quad \exp\left(-\dfrac{1}{2}\text{tr}\left({\Sigma^{(c)}}^{-1}(\textbf{U}-\mu)({\Sigma^{(r)}}^{-1})'(\textbf{U}-\mu)'\right)\right),
\end{align}
where $\Sigma^{(c)}$ and $\Sigma^{(r)}$ are two symmetric positive-definite matrices that denote the covariance across columns and rows of \textbf{U}, respectively. We use the notation $\textbf{U}\sim\mathcal{N}(\mu, \Sigma^{(c)} \otimes \Sigma^{(r)})$. 

We now spend some words to elucidate the density function above. The vectorization operator `vec' transforms a matrix $\textbf{U}\in\mathbb{R}^{d\times n}$ into  the column consisting of the column vectors of $\textbf{U}$, i.e.  $\text{vec}(\textbf{U})=(u_{.1}'\ldots,u_{.n}')'$, where $u_{.j}$ is the $j$-th column of the matrix $\textbf{U}$. The mean and covariance of the matrix \textbf{U} is defined to be the mean and covariance of this random vector $\text{vec}(\textbf{U})$, which is a multivariate normal in our case and we write $\text{vec}(\textbf{U})\sim \mathcal{N}_{dn}(\text{vec}(\mu),\Sigma^{(c)}\otimes\Sigma^{(r)})$. 
The notation significantly simplifies if the columns of $\{u_{.j}\}_{j=1}^n$ are $n$ i.i.d. observations with a shared covariance matrix $\Sigma_U$, and in that case $\text{cov}(\textbf{U})=\Sigma_{U}\otimes I_n$. In such cases, a key result that is used repeatedly is provided by the following proposition (see \cite{vsmidl2006variational}). 
\begin{Proposition}\label{lem}
Let \textbf{U} be a given $d\times n$ matrix valued random variable. If $\textbf{U}\sim\mathcal{N}(\mu,I_d\otimes \Sigma),$ then $\textbf{U}'\sim \mathcal{N}(\mu',\Sigma\otimes I_d)$ and the second non-central moments are given by 
$$\mathbb{E}(\textbf{U}'\Pi_1\textbf{U})= \text{tr}\left(\Pi_1 \right)\Sigma+\mu'\Pi_1\mu,$$
$$\mathbb{E}(\textbf{U}\Pi_2\textbf{U}')= \text{tr}\left(\Pi_2\Sigma\right)I_d+\mu\Pi_2\mu',$$
where $\Pi_1, \ \Pi_2$ are $d\times d$ and $n\times n$ symmetric positive definite matrices, respectively. 
\end{Proposition}
For two $d\times k$ matrix normal distributions $q_{\textbf{W}}=\mathcal{N}(\mu_{\textbf{W}},I_d\otimes \Sigma_\textbf{W})$ and $q^*_{\textbf{W}}=\mathcal{N}(\mu^*_{\textbf{W}},I_d\otimes \Sigma^*_\textbf{W})$, we define the quantity:
$$\mathcal{J}(\Sigma_\textbf{W},
\Sigma_\textbf{W}^*)=d\left(\text{tr}\left([\Sigma^*_{\textbf{W}}]^{-1}\Sigma_{\textbf{W}}\right)-k-\log\text{det}\left([\Sigma^*_{\textbf{W}}]^{-1}\Sigma_{\textbf{W}}\right)\right).$$
We now record the following expression of KL--divergence between $q_{\textbf{W}}$ and $q^*_{\textbf{W}}$:
\begin{equation}\label{matrix_kl_def}
\text{KL}\left(q_{\textbf{W}}\|q^*_{\textbf{W}}\right)= \dfrac{1}{2}\left(\mathcal{J}(\Sigma_\textbf{W},\Sigma_\textbf{W}^*)+\left(\text{vec}(\mu_{\textbf{W}})-\text{vec}(\mu^*_{\textbf{W}})\right)'\left(I_d\otimes \Sigma^*_{\textbf{W}}\right)^{-1}\left(\text{vec}(\mu_{\textbf{W}})-\text{vec}(\mu^*_{\textbf{W}})\right) \right).
\end{equation}
\section{Details of the derivation of log-posterior}\label{ap2}
\noindent In this appendix, our goal is to derive the log-posterior $\log \pi(\textbf{W},\textbf{Z})$. We recall the BPCA model with normal error terms in terms of matrix-valued random variables as:
\begin{equation}\label{model}
\textbf{X}=\textbf{Z}\textbf{W}'+\textbf{E},
\end{equation}
where the rows of the $n\times d$ data matrix $\mathbf{X}$ are observations $x_i$, the $n \times k$ matrix $\mathbf{Z}$ contains latent variables $z_i$, and $\mathbf{E}$ is an $n \times d$ noise matrix whose rows are scaled Gaussian noise vectors $\tau_0^{-1/2} \epsilon_i$ for $i=1,2,\ldots,n$. The loading matrix $\mathbf{W}$ has size $d \times k$, and we fix the noise precision $\tau_0 > 0$ for simplicity. The likelihood and priors are given by (see \cite{bishop1}):
\[
p(\textbf{X}|\textbf{W},\textbf{Z})=\mathcal{N}\left(\textbf{X};\textbf{Z}\textbf{W}',\tau_0^{-1}I_n\otimes I_d\right), \quad
p(\textbf{W})=\mathcal{N}(\textbf{W};\textbf{0},I_d\otimes \Lambda^{-1}), \quad
p(\textbf{Z})=\mathcal{N}(\textbf{Z};\textbf{0},I_n\otimes I_k),
\]
where $\Lambda = \text{diag}(\sigma_1^{-2}, \ldots, \sigma_k^{-2})$ is a diagonal matrix of precision hyperparameters. Assume that, $\Lambda\not\propto I_k$. In the above equation $\mathcal{N}\left(\textbf{X};\textbf{Z}\textbf{W}',\tau_0^{-1}I_n\otimes I_d\right)$ denotes a matrix normal density and an interpretation follows in \autoref{ap1}. The posterior density is thus:
\begin{equation}\label{joint}
\pi(\textbf{W},\textbf{Z})\propto p(\textbf{X}|\textbf{W},\textbf{Z})p(\textbf{W})p(\textbf{Z}).
\end{equation}
Therefore, given the hyper-parameter $\Lambda$, the logarithm of the posterior can be obtained as:
\begin{align}
\nonumber
\log \pi(\textbf{W},\textbf{Z})&=\log p(\textbf{X}|\textbf{W},\textbf{Z})+\log p(\textbf{W})+\log p(\textbf{Z})+\text{cst}\\
\nonumber
&=\dfrac{dn}{2}\log\tau_0-\dfrac{1}{2}\tau_0\text{tr}\left((\textbf{X}-\textbf{Z}\textbf{W}')(\textbf{X}-\textbf{Z}\textbf{W}')'\right)-\dfrac{1}{2}\text{tr}(\textbf{W}\Lambda\textbf{W}')-\dfrac{1}{2}\text{tr}(\textbf{Z}\textbf{Z}')+\text{cst}\\
\nonumber
&=-\dfrac{1}{2}\tau_0\text{tr}\left(\textbf{X}\textbf{X}'-\textbf{X}\textbf{W}\textbf{Z}'-\textbf{Z}\textbf{W}'\textbf{X}'+\textbf{Z}\textbf{W}'\textbf{W}\textbf{Z}'\right)-\dfrac{1}{2}\text{tr}(\textbf{W}\Lambda\textbf{W}') -\dfrac{1}{2}\text{tr}(\textbf{Z}\textbf{Z}')+\text{cst}\\
\nonumber
&=\tau_0\text{tr}\left(\textbf{W}\textbf{Z}'\textbf{X}\right)-\dfrac{1}{2}\tau_0\text{tr}\left(\textbf{W}\textbf{Z}'\textbf{Z}\textbf{W}'\right)-\dfrac{1}{2}\text{tr}(\textbf{W}\Lambda\textbf{W}') -\dfrac{1}{2}\text{tr}(\textbf{Z}\textbf{Z}')+\text{cst},
\end{align}
where at the final step we used properties of the trace operator and the constant term does not depend on the parameters $\textbf{W},$ $\textbf{Z}$.
\section{Details of the derivation of ELBO}\label{elbo_derivation}
Let $q=(\mu_\textbf{W},\Sigma_{\textbf{W}},\mu_\textbf{Z},\Sigma_{\textbf{Z}})\in \mathbb{R}^{d}\times \mathbb{S}_k^+\times \mathbb{R}^{n}\times \mathbb{S}_k^+$ be given. To derive an explicit expression of $\text{ELBO}(q)$, we first recall the definition 
$$\text{ELBO}(q)=\mathbb{E}_q[\log p(\textbf{W},\textbf{Z},\textbf{X})]-\mathbb{E}_q[\log q],$$
where $q_\textbf{W}=\mathcal{N}( \mu_\textbf{W},I_d\otimes \Sigma_\textbf{W})$, $q_\textbf{Z}=\mathcal{N}(\mu_\textbf{Z}, I_n\otimes \Sigma_\textbf{Z})$ and $q= q_\textbf{W}q_\textbf{Z}$. Expanding we have that
$$\text{ELBO}(q)= \mathbb{E}_{q}[\log \pi(\textbf{W}, \textbf{Z})]+ \mathbb{E}_{q}[-\log q_\textbf{W}]+ \mathbb{E}_{q}[-\log q_\textbf{Z}]+\text{cst}.$$
The entropy terms in the above expressions are $d/2\log\det(\Sigma_{\textbf{W}})$ and $n/2\log\det(\Sigma_{\textbf{Z}})$, respectively. Using \eqref{elbo_technical} in the proof of \autoref{delta_upperbound}, we conclude that 
\begin{align*}
\text{ELBO}(q) &= \tau_0\text{tr}\left(\mu_{\textbf{W}}\mu_{\textbf{Z}}'\textbf{X}\right) 
-\dfrac{\tau_0}{2}\text{tr}\left(\Gamma_\textbf{W}\Gamma_{\textbf{Z}}\right) 
-\dfrac{d}{2}\text{tr}\left(\Lambda \Sigma_{\textbf{W}}\right) 
-\dfrac{1}{2}\text{tr}\left(\mu_\textbf{W}\Lambda\mu_{\textbf{W}}'\right) \\
&\quad -\dfrac{1}{2}\text{tr}\left(\Gamma_{\textbf{Z}}\right) 
+\dfrac{d}{2}\log\det(\Sigma_{\textbf{W}}) 
+\dfrac{n}{2}\log\det(\Sigma_{\textbf{Z}}) + \text{cst},
\end{align*}
where 
$$\Gamma_{\textbf{W}}=d\Sigma_{\textbf{W}}+\mu_{\textbf{W}}'\mu_{\textbf{W}}\text{  and  }\Gamma_{\textbf{Z}}=n\Sigma_{\textbf{Z}}+\mu_{\textbf{Z}}'\mu_{\textbf{Z}}.$$
\section{Details of the derivation of \autoref{algo1}}\label{ap3}
\subsection{Identification of $q^{(t+1)}_\textbf{W}$}\label{apb2} 
We have that
\begin{align}
\nonumber
q^{(t+1)}_{\textbf{W}}&\propto \exp\left(\int q_{\textbf{Z}}^{(t)}(\textbf{Z})\log p(\textbf{W},\textbf{Z},\textbf{X})\ \d\textbf{Z}\right)\\ 
\nonumber
&\propto\exp\Bigg[-\dfrac{1}{2}\text{tr}\left(-2\textbf{W}\int q_{\textbf{Z}}^{(t)}(\textbf{Z})(\tau_0\textbf{Z}'\textbf{X})\ \d\textbf{Z}\right)\\
\nonumber
&\quad-\dfrac{1}{2}\text{tr}\left(\textbf{W}\left(\int q_{\textbf{Z}}^{(t)}(\textbf{Z})(\tau_0\textbf{Z}'\textbf{Z})\ \d\textbf{Z}\right)\textbf{W}'\right)-\dfrac{1}{2}\text{tr}(\textbf{W}\Lambda\textbf{W}')\Bigg]\\
\nonumber
&= \exp\Bigg[-\dfrac{1}{2}\text{tr}\left(-2\textbf{W}\int q_\mathbf{X}^{(t)}(\textbf{Z})(\tau_0\textbf{Z}'\textbf{X})\ \d\textbf{Z}\right)\\
\nonumber
&\quad -\dfrac{1}{2}\text{tr}\left(\textbf{W}\left(\int q_{\textbf{Z}}^{(t)}(\textbf{Z})(\tau_0\textbf{Z}'\textbf{Z})\ \d\textbf{Z}+\Lambda\right)\textbf{W}'\right)\Bigg].
\end{align}
Comparing the above exponentiated quadratic form in $\textbf{W}$, we infer that $q^{(t+1)}_{\textbf{W}}$ is a matrix normal distribution of the form $\mathcal{N}(\mu^{(t+1)}_{\textbf{W}},I_{d}\otimes \Sigma^{(t+1)}_{\textbf{W}}),$ where the parameters are given by
\begin{align}
\nonumber
\Sigma^{(t+1)}_{\textbf{W}}&=(\tau_0(n\Sigma^{(t)}_{\textbf{Z}}+(\mu^{(t)}_{\textbf{Z}})'\mu^{(t)}_{\textbf{Z}})+\Lambda)^{-1},
\end{align}
using \autoref{lem} and independence between \textbf{Z} and $\tau$, and comparing the coefficient of \textbf{W} inside the trace factor, we obtain
\begin{align}
\nonumber
\mu^{(t+1)}_{\textbf{W}}&=(\tau_0\textbf{X}'\mu_{\textbf{Z}}^{(t)})\Sigma^{(t+1)}_{\textbf{W}}.
\end{align}
The parameters $\mu_\textbf{Z}^{(t)}$ and $\Sigma_\textbf{Z}^{(t)}$ are defined below.
\subsection{Identification of $q^{(t+1)}_{\textbf{Z}}$}
We have that
\begin{align}
\nonumber
q^{(t+1)}_{\textbf{Z}}&\propto \exp\left(\int q_\mathbf{W}^{(t+1)}(\textbf{W})\log p(\textbf{W},\textbf{Z},\textbf{X})\ \d\textbf{W}\right)\\ 
\nonumber
&\propto\exp\Bigg[-\dfrac{1}{2}\text{tr}\left(-2\int q_\mathbf{W}^{(t+1)}(\textbf{W})(\tau_0\textbf{W}\textbf{Z}'\textbf{X})\ \d\textbf{W}\right)\\
\nonumber
&\quad-\dfrac{1}{2}\text{tr}\left(\textbf{Z}\left(\int q_\mathbf{W}^{(t+1)}(\textbf{W})(\tau_0\textbf{W}'\textbf{W})\ \d\textbf{W}\right)\textbf{Z}'\right)-\dfrac{1}{2}\text{tr}(\textbf{Z}\textbf{Z}')\Bigg]\\
\nonumber
&=\exp\Bigg[-\dfrac{1}{2}\text{tr}\left(-2\int q_\mathbf{W}^{(t+1)}(\textbf{W})(\tau_0\textbf{W}\textbf{Z}'\textbf{X})\ \d\textbf{W}\right)\\
\nonumber
&\quad-\dfrac{1}{2}\text{tr}\left(\textbf{Z}\left(\int q_\mathbf{W}^{(t+1)}(\textbf{W})(\tau_0\textbf{W}'\textbf{W})\ \d\textbf{W}+I_k\right)\textbf{Z}'\right)\Bigg],
\end{align}
where we used $\text{tr}(MN)=\text{tr}(NM)$ for matrices $M\in\mathbb{R}^{m\times n}$ and $N\in\mathbb{R}^{n\times m}$. As before, comparing the above exponentiated quadratic form in $\textbf{Z}$ we infer that $q^{(t+1)}_{\textbf{Z}}$ is a matrix normal distribution of the form $\mathcal{N}(\mu^{(t+1)}_{\textbf{Z}},I_{n}\otimes \Sigma^{(t+1)}_{\textbf{Z}}),$ where the parameters are given by
\begin{align}
\nonumber
\Sigma^{(t+1)}_{\textbf{Z}}&=(\tau_0(d\Sigma^{(t+1)}_{\textbf{W}}+(\mu^{(t+1)}_{\textbf{W}})'\mu^{(t+1)}_{\textbf{W}})+I_k)^{-1},
\end{align}
using \autoref{lem} and independence between \textbf{W} and $\tau$, and comparing the coefficient of \textbf{Z} inside the trace factor, we obtain
\begin{align}
\nonumber
\mu^{(t+1)}_{\textbf{Z}}&=
(\tau_0\textbf{X}\mu_{\textbf{W}}^{(t+1)})\Sigma^{(t+1)}_{\textbf{Z}}.
\end{align}
\section{Auxiliary Lemmas}
\begin{Lemma}\label{rate_lemma}
    Let $\{x_n\}\subset \mathbb{R}^\ell$ be a sequence of vectors such that for all $n\in\mathbb{N}$, we have $\|x_n-x\|\leqslant cr^n$  for some positive constant $c$ and $r\in(0,1)$ fixed. Then we have that
    $$\left\|\dfrac{x_n}{\|x_n\|}-\dfrac{x}{\|x\|}\right\|\leqslant \dfrac{2cr^n}{\|x\|},$$
where $x\neq 0$.
\end{Lemma}
\begin{proof}
    \begin{align*}
\left\|\frac{x_n}{\|x_n\|}-\frac{x}{\|x\|}\right\|
&=\left\|\frac{x_n\left(\|x\|-\|x_n\|\right)+\|x_n\|\left(x_n-x\right)}{\|x_n\|\cdot\|x\|}\right\|\\[1ex]
&\leqslant \frac{1}{\|x\|}|\|x_n\|-\|x\||+\frac{1}{\|x\|}\|x_n-x\|\\[1ex]
&\leqslant\frac{2}{\|x\|}\|x_n-x\|.
\end{align*}

\end{proof}
\begin{Lemma}\label{delta_upperbound}
Let $q_{\textbf{W}},q_{\textbf{Z}}, q^*_{\textbf{W}},q^*_{\textbf{Z}}$ denote the matrix normal distributions  $\mathcal{N}( \mu_\textbf{W}, I_d\otimes \Sigma_\textbf{W}),\mathcal{N}( \mu_\textbf{Z}, I_n\otimes \Sigma_\textbf{Z})$, and $\mathcal{N}( \mu^*_{\textbf{W}}, I_d\otimes \Sigma^*_\textbf{W}),\mathcal{N}( \mu^*_{\textbf{Z}},I_n\otimes \Sigma^*_{\textbf{Z}})$. Let us denote the quantities $A_1=\|\mu_{\textbf{Z}}-\mu^*_{\textbf{Z}}\|, A_2=\|\mu_{\textbf{W}}-\mu^*_{\textbf{W}}\|,B_1=\|\Sigma_{\textbf{Z}}-\Sigma^*_{\textbf{Z}}\|,B_2=\|\Sigma_{\textbf{W}}-\Sigma^*_{\textbf{W}}\|,C_1=\|\mu_{\textbf{Z}}\|+\|\mu^*_{\textbf{Z}}\|,C_2=\|\mu_{\textbf{W}}\|+\|\mu^*_{\textbf{W}}\|$. We have that
    $$|\Delta_{q^*}(q_{\textbf{W}},q_{\textbf{Z}})|\leqslant  \left(\tau_0\|\textbf{X}\|+\dfrac{C_1C_2\tau_0}{2}\right)A_1A_2+(nd\tau_0B_1B_2+dC_1\tau_0 B_2A_1+ nC_2\tau_0 A_2B_1)/2,$$
where the quantity $\Delta_{q^*}(q_\textbf{W},q_\textbf{Z})$ is defined in \eqref{covariance}.
\end{Lemma}
\begin{proof}
For any two non-negative integrable functions $f(\textbf{W}),\ g(\textbf{Z})$, we define the following functional
$$\delta\left(f,g\right)=\iint f(\textbf{W})g(\textbf{Z})\log \pi(\textbf{W},\textbf{Z})\ \d\textbf{W} \ \d\textbf{Z},$$
provided the quantity is finite. By definition, we have that
\begin{align}
\Delta_{q^*}(q_{\textbf{W}},q_{\textbf{Z}})\nonumber
&=\iint \left(q_\textbf{W}(\textbf{W})-q^*_{\textbf{W}}(\textbf{W})\right)\left(q_\textbf{Z}(\textbf{Z})-q^*_{\textbf{Z}}(\textbf{Z})\right)\log \pi(\textbf{W},\textbf{Z}) \ \d\textbf{W} \ \d\textbf{Z}\\ 
\label{delta}
&=\delta\left(q_{\textbf{W}},q_{\textbf{Z}}\right)-\delta\left(q^*_{\textbf{W}},q_{\textbf{Z}}\right)-\delta\left(q_{\textbf{W}},q^*_{\textbf{Z}}\right)+\delta\left(q^*_{\textbf{W}},q^*_{\textbf{Z}}\right).
\end{align}
 Let us now focus on
\begin{align}
\nonumber
\delta\left(q_{\textbf{W}},q_{\textbf{Z}}\right)&= \mathbb{E}[\log p(\textbf{X}|\textbf{W},\textbf{Z})]+\mathbb{E}[\log p(\textbf{W})]+ \mathbb{E}[\log p(\textbf{Z})]+\text{cst}\\ 
\nonumber
&=\tau_0\text{tr}\left(\mu_{\textbf{W}}\mu_{\textbf{Z}}'\textbf{X}\right)-\dfrac{\tau_0}{2}\mathbb{E}\left[\text{tr}\left(\textbf{W}\textbf{Z}'\textbf{Z}\textbf{W}\right)\right]-\dfrac{1}{2} \mathbb{E}[\text{tr}(\textbf{W}\Lambda\textbf{W}')]-\dfrac{1}{2}\mathbb{E}\left[\text{tr}\left(\textbf{Z}\textbf{Z}'\right)\right]+ \text{cst}.
\end{align}
Detailed derivation of the log-posterior can be found in \autoref{ap2}. To compute the the expectations of the above quadratic forms, let us define the following quantities for convenience
$$\Gamma_{\textbf{W}}=d\Sigma_{\textbf{W}}+\mu_{\textbf{W}}'\mu_{\textbf{W}}\text{  and  }\Gamma_{\textbf{Z}}=n\Sigma_{\textbf{Z}}+\mu_{\textbf{Z}}'\mu_{\textbf{Z}}.$$
Now from \autoref{expectation_lemmas} it follows that
$$\mathbb{E}[\text{tr}\left(\textbf{W}\textbf{Z}'\textbf{Z}\textbf{W}'\right)]= \text{tr}\left(\Gamma_\textbf{W}\Gamma_{\textbf{Z}}\right),
$$
$$
   \mathbb{E}[\text{tr}(\textbf{W}\Lambda\textbf{W}')]=d\text{tr}\left(\Lambda \Sigma_{\textbf{W}}\right)+\text{tr}\left(\mu_\textbf{W}\Lambda\mu_{\textbf{W}}'\right).
$$
Therefore, we have that
\begin{equation}\label{elbo_technical}
\mathbb{E}[\log \pi(\textbf{W}, \textbf{Z})]=\tau_0\text{tr}\left(\mu_{\textbf{W}}\mu_{\textbf{Z}}'\textbf{X}\right)-\dfrac{\tau_0}{2}\text{tr}\left(\Gamma_\textbf{W}\Gamma_{\textbf{Z}}\right)-\dfrac{d}{2}\text{tr}\left(\Lambda \Sigma_{\textbf{W}}\right)-\dfrac{1}{2}\text{tr}\left(\mu_\textbf{W}\Lambda\mu_{\textbf{W}}'\right)-\dfrac{1}{2}\text{tr}\left(\Gamma_{\textbf{Z}}\right)+\text{cst}.
\end{equation}
Combining everything in \eqref{delta} (note that only the cross-terms survive),
\begin{align*}
\Delta_{q^*}(q_{\textbf{W}},q_{\textbf{Z}}) 
&= \delta\left(q_{\textbf{W}},q_{\textbf{Z}}\right)-\delta\left(q^*_{\textbf{W}},q_{\textbf{Z}}\right)-\delta\left(q_{\textbf{W}},q^*_{\textbf{Z}}\right)+\delta\left(q^*_{\textbf{W}},q^*_{\textbf{Z}}\right) \\  
&=\tau_0\text{tr}\left(\left(\mu_{\textbf{W}}\mu_{\textbf{Z}}'-\mu^*_{\textbf{W}}\mu_{\textbf{Z}}'-\mu_{\textbf{W}}[\mu^*_{\textbf{Z}}]'+\mu^*_{\textbf{W}}[\mu^*_{\textbf{Z}}]'\right)\textbf{X}\right) \\ 
&\quad -\dfrac{\tau_0}{2}\text{tr}\left(\left(\Gamma_{\textbf{W}}\Gamma_{\textbf{Z}}'-\Gamma_{\textbf{W}}^*\Gamma_{\textbf{Z}}'-\Gamma_{\textbf{W}}[\Gamma_{\textbf{Z}}^*]'+\Gamma_{\textbf{W}}^*[\Gamma_{\textbf{Z}}^*]'\right)\right) \\ 
&= \tau_0 \text{tr}\left(\left(\mu_{\textbf{W}}-\mu^*_{\textbf{W}}\right)\left(\mu'_{\textbf{Z}}-[\mu^*_{\textbf{Z}}]'\right)\textbf{X}\right) \\ 
&\quad -\dfrac{\tau_0}{2}\text{tr}\left(\left(\Gamma_{\textbf{W}}-\Gamma_{\textbf{W}}^*\right)\left(\Gamma'_{\textbf{Z}}-[\Gamma^*_{\textbf{Z}}]'\right)\right)
\end{align*}
Using \autoref{trace_lemma} and the submultiplicativity of the Frobenius norm, we get that
\begin{equation}\label{delta_technical}
|\Delta_{q^*}(q_{\textbf{W}},q_{\textbf{Z}})|\leqslant \tau_0 \|\textbf{X}\|.\|\mu_{\textbf{W}}-\mu^*_{\textbf{W}}\|.\|\mu_{\textbf{Z}}-\mu^*_{\textbf{Z}}\|+(\tau_0/2)\|\Gamma_{\textbf{W}}-\Gamma^*_{\textbf{W}}\|.\|\Gamma_{\textbf{Z}}-\Gamma^*_{\textbf{Z}}\|.
\end{equation}
We can simplify the right side further. Note that by triangle inequality
\begin{align}
\|\Gamma_{\textbf{W}}-\Gamma^*_{\textbf{W}}\|\nonumber
&\leqslant d\|\Sigma_{\textbf{W}}-\Sigma^*_{\textbf{W}}\|+\|\mu'_{\textbf{W}}\mu_{\textbf{W}}-[\mu^*_{\textbf{W}}]'\mu^*_{\textbf{W}}\|\\ 
\nonumber
&\leqslant d\|\Sigma_{\textbf{W}}-\Sigma^*_{\textbf{W}}\|+(\|\mu_{\textbf{W}}\|+\|\mu^*_{\textbf{W}}\|)\cdot\|\mu_{\textbf{W}}-\mu^*_{\textbf{W}}\|,
\end{align}
since $\mu'_{\textbf{W}}\mu_{\textbf{W}}-[\mu^*_{\textbf{W}}]'\mu^*_{\textbf{W}}=\mu'_{\textbf{W}}(\mu_{\textbf{W}}-\mu^*_{\textbf{W}})+(\mu'_{\textbf{W}}-[\mu^*_{\textbf{W}}]')\mu^*_{\textbf{W}}$, and analogously,
\begin{align}
\|\Gamma_{\textbf{Z}}-\Gamma^*_{\textbf{Z}}\|\nonumber
&\leqslant n\|\Sigma_{\textbf{Z}}-\Sigma^*_{\textbf{Z}}\|+\|\mu'_{\textbf{Z}}\mu_{\textbf{Z}}-\mu'_{\textbf{Z}^*}\mu^*_{\textbf{Z}}\|\\ 
\nonumber
&\leqslant n\|\Sigma_{\textbf{Z}}-\Sigma^*_{\textbf{Z}}\|+(\|\mu_{\textbf{Z}}\|+\|\mu^*_{\textbf{Z}}\|)\cdot\|\mu_{\textbf{Z}}-\mu^*_{\textbf{Z}}\|.
\end{align}
Combining everything in \eqref{delta_technical} we obtain
\begin{align}
\nonumber
|\Delta_{q^*}(q_{\textbf{W}},q_{\textbf{Z}})|
&\leqslant \tau_0\|\textbf{X}\|A_1A_2+(\tau_0/2)\left(dB_2+C_2A_2\right)\left(nB_1+C_1A_1\right)\\ 
\nonumber
&=  \tau_0\|\textbf{X}\| A_1A_2+ (C_1C_2\tau_0A_1A_2+nd\tau_0B_1B_2+dC_1\tau_0 B_2A_1+ nC_2\tau_0 A_2B_1)/2\\
\nonumber
&=  \left(\tau_0\|\textbf{X}\|+\dfrac{C_1C_2\tau_0}{2}\right)A_1A_2+(nd\tau_0B_1B_2+dC_1\tau_0 B_2A_1+ nC_2\tau_0 A_2B_1)/2.
\end{align}
\end{proof}
\begin{Lemma}\label{KL_lowerbound}
Let $r_0>0$ be given. For any $(q_{\textbf{W}},q_{\textbf{Z}})\in \mathcal{B}^*_{\textbf{W}}(r_0)\times \mathcal{B}^*_{\textbf{Z}}(r_0)$, we have the following lower bounds
$$D_{\text{KL},1/2}\left(q_{\textbf{W}}\|q^*_{\textbf{W}}\right)\geqslant \dfrac{dc_3}{\|\Sigma^*_{\textbf{W}}\|^2 e^{1+2(r_0/d)}}\|\Sigma_{\textbf{W}}-\Sigma^*_{\textbf{W}}\|^2+\dfrac{c_2}{\|\Sigma^*_{\textbf{W}}\|}\|\mu_{\textbf{W}}-\mu^*_{\textbf{W}}\|^2,$$
$$D_{\text{KL},1/2}\left(q_{\textbf{Z}}\|q^*_{\textbf{Z}}\right)\geqslant  \dfrac{nc_3}{\|\Sigma^*_{\textbf{Z}}\|^2 e^{1+2(r_0/n)}}\|\Sigma_{\textbf{Z}}-\Sigma^*_{\textbf{Z}}\|^2+\dfrac{c_2'}{\|\Sigma^*_{\textbf{Z}}\|}\|\mu_{\textbf{Z}}-\mu^*_{\textbf{Z}}\|^2, $$
where $c_2=1/4(1+\exp{\left(-(1+(2r_0/d))\right)}),\ c_2'=1/4(1+\exp{\left(-(1+(2r_0/n))\right)})$ and $c_3=1/2$. 
\end{Lemma}
\begin{proof}
We only prove the first lower bound as the other follows analogously. Let us write $\mathcal{P}_{\textbf{W}^*}=I_d\otimes \Sigma^*_{\textbf{W}}$. Therefore, by definition of KL--divergence between two matrix normal variables in \eqref{matrix_kl_def}, we have that
\begin{align*}
\text{KL}\left(q_{\textbf{W}}\|q^*_{\textbf{W}}\right)&= \dfrac{1}{2}\left(\mathcal{J}(\Sigma_\textbf{W}, \Sigma_\textbf{W}^*)+\left(\text{vec}(\mu_{\textbf{W}})-\text{vec}(\mu^*_{\textbf{W}})\right)'\mathcal{P}^{-1}_{\textbf{W}^*}\left(\text{vec}(\mu_{\textbf{W}})-\text{vec}(\mu^*_{\textbf{W}})\right) \right)\\
&\geqslant \dfrac{1}{2}\mathcal{J}(\Sigma_\textbf{W}, \Sigma_\textbf{W}^*)+\dfrac{1}{2} \lambda_{\text{min}}\left(\mathcal{P}^{-1}_{\textbf{W}^*}\right)\|\left(\text{vec}(\mu_{\textbf{W}})-\text{vec}(\mu^*_{\textbf{W}})\right)\|_2^2\\
&= \dfrac{1}{2}\mathcal{J}(\Sigma_\textbf{W}, \Sigma_\textbf{W}^*)+\dfrac{1}{2\|\Sigma^*_{\textbf{W}}\|}_{\text{op}}\|\mu_{\textbf{W}}-\mu^*_{\textbf{W}}\|^2.
\end{align*}
 Analogous inequality also holds for $\text{KL}\left(q^*_{\textbf{W}}\|q_{\textbf{W}}\right)$. Thus we arrive at the following follower bound
\begin{align*}
D_{\text{KL},1/2}\left(q_{\textbf{W}}\|q^*_{\textbf{W}}\right)&= \dfrac{1}{2}\left[\text{KL}\left(q_{\textbf{W}}\|q^*_{\textbf{W}}\right)+\text{KL}\left(q^*_{\textbf{W}}\|q_{\textbf{W}}\right)\right]\\
&\geqslant \dfrac{1}{2}\left(\mathcal{J}(\Sigma_\textbf{W}, \Sigma_\textbf{W}^*)+\mathcal{J}(\Sigma^*_\textbf{W},\Sigma_\textbf{W})\right)+\dfrac{1}{4}\left(\dfrac{1}{\|\Sigma^*_{\textbf{W}}\|_{\text{op}}}+\dfrac{1}{\|\Sigma_{\textbf{W}}\|_{\text{op}}}\right)\|\mu_{\textbf{W}}-\mu_{\textbf{W}}\|^2\\
&\geqslant \dfrac{1}{2}\left(\mathcal{J}(\Sigma_\textbf{W}, \Sigma_\textbf{W}^*)+\mathcal{J}(\Sigma^*_\textbf{W},\Sigma_\textbf{W})\right)\\
&\quad +\dfrac{1}{4}\left(\dfrac{1}{\|\Sigma^*_{\textbf{W}}\|}+\dfrac{1}{\|\Sigma^*_{\textbf{W}}\|e^{1+2(r_0/d)}}\right)\|\mu_{\textbf{W}}-\mu_{\textbf{W}}\|^2 \quad\text{(by \autoref{psd_norm_lemma} and \autoref{subset_inclusion})}\\
&= \dfrac{1}{2}\left(\mathcal{J}(\Sigma_\textbf{W}, \Sigma_\textbf{W}^*)+\mathcal{J}(\Sigma^*_\textbf{W},\Sigma_\textbf{W})\right) +\dfrac{c_2}{\|\Sigma^*_{\textbf{W}}\|}\|\mu_{\textbf{W}}-\mu^*_{\textbf{W}}\|^2.
\end{align*}
To finish the proof we focus on the term $\mathcal{J}(\Sigma_\textbf{W}, \Sigma_\textbf{W}^*)$. We have that
\begin{align*}
\dfrac{1}{2}\left(\mathcal{J}(\Sigma_\textbf{W}, \Sigma_\textbf{W}^*)+\mathcal{J}(\Sigma^*_\textbf{W},\Sigma_\textbf{W})\right)&= \dfrac{d}{2}\left(\text{tr}\left([\Sigma^*_{\textbf{W}}]^{-1}\Sigma_{\textbf{W}}\right)+\text{tr}\left(\Sigma^{-1}_{\textbf{W}}\Sigma^*_{\textbf{W}}\right)-2k\right)\\
&\geqslant  \dfrac{d}{2\|\Sigma^*_{\textbf{W}}\|_{\text{op}}\|\Sigma_{\textbf{W}}\|_{\text{op}}}\|\Sigma_{\textbf{W}}-\Sigma^*_{\textbf{W}}\|^2\quad\text{(by \autoref{florians_lemma})}\\
&\geqslant  \dfrac{dc_3}{\|\Sigma^*_{\textbf{W}}\|^2 e^{1+2(r_0/d)}}\|\Sigma_{\textbf{W}}-\Sigma^*_{\textbf{W}}\|^2\quad\text{(by \autoref{psd_norm_lemma} and \autoref{control_lemma})}
\end{align*}
The second part of the Lemma follows analogously. 
\end{proof}
\begin{Lemma}\label{subset_inclusion}
For every $r_0>0$, the following inclusions hold
$$\mathcal{B}^*_{\textbf{W}}\left(r_0\right)\subseteq \left\{(\mu_{\textbf{W}},\Sigma_{\textbf{W}}): \|\Sigma_{\textbf{W}}\|_{\text{op}} \leqslant \|\Sigma^*_{\textbf{W}}\|e^{1+2(r_0/d)},\ \|\mu_\textbf{W}-\mu^*_{\textbf{W}}\|^2\leqslant 2r_0 e^{(1+2(r_0/d))}\|\Sigma^*_{\textbf{W}}\|_{\text{op}}\right\},$$
$$\mathcal{B}^*_{\textbf{Z}}\left(r_0\right)\subseteq \left\{(\mu_{\textbf{Z}},\Sigma_{\textbf{Z}}): \|\Sigma_{\textbf{Z}}\|_{\text{op}} \leqslant \|\Sigma^*_{\textbf{Z}}\|e^{1+2(r_0/n)},\ \|\mu_\textbf{Z}-\mu^*_{\textbf{Z}}\|^2\leqslant 2r_0 e^{(1+2(r_0/n))}\|\Sigma^*_{\textbf{Z}}\|_{\text{op}}\right\}.$$
\end{Lemma}
\begin{proof}
We only prove the first inclusion. Recall the expression for KL divergence for matrix normal distributions in \eqref{matrix_kl_def}. The term $\mathcal{J}(\Sigma_{\textbf{W}},\Sigma^*_{\textbf{W}})$ can be re-written in the following way (see \cite{bhattacharya2023convergence}):
$$\text{KL}\left(q^*_{\textbf{W}}\|q_{\textbf{W}}\right)=\dfrac{d}{2}\sum\limits_{j=1}^k \psi\left(\lambda_j(\Delta)\right) + \dfrac{1}{2}\left(\text{vec}(\mu_{\textbf{W}})-\text{vec}(\mu^*_{\textbf{W}})\right)'(I_d\otimes \Sigma_{\textbf{W}})^{-1}\left(\text{vec}(\mu_{\textbf{W}})-\text{vec}(\mu^*_{\textbf{W}})\right),$$
where $\psi(x) = x-\log x-1$, $\Delta = \Sigma_{\textbf{W}}^{-1/2}\Sigma^*_{\textbf{W}}\Sigma_{\textbf{W}}^{-1/2}$ and $\{\lambda_{j}(\Delta)\}_{j=1}^d$ denote the eigenvalues in the sorted order. Therefore, $\text{KL}\left(q^*_{\textbf{W}}\|q_{\textbf{W}}\right)\leqslant r_0$ implies
$$\psi\left(\lambda_j(\Delta)\right)\leqslant 2r_0/d \quad\text{for }j=1,\ldots,d,$$
$$\left(\text{vec}(\mu_{\textbf{W}})-\text{vec}(\mu^*_{\textbf{W}})\right)'(I_d\otimes \Sigma_{\textbf{W}})^{-1}\left(\text{vec}(\mu_{\textbf{W}})-\text{vec}(\mu^*_{\textbf{W}})\right)\leqslant 2r_0.$$
We have $\psi(x)\geqslant 0$ for every $x\in(0,\infty)$ with equality if and only if $x=1$. Since $\lim_{x\to\infty}\psi(x)=\lim_{x\to 0}\psi(x)=\infty$, from \autoref{psi_bound} it follows that for any $c>0$, there exists an interval $[\ell_c,u_c]$ containing $1$ such that $\psi(x)\leqslant c$ if and only if $x\in [\ell_c, u_c]$. Using the same lemma, we also have that $\ell_c\in (e^{-(1+c)},e^{-c}).$ Thus $\psi\left(\lambda_j(\Delta)\right)\leqslant 2r_0/d$ implies $\lambda_j(\Delta)\geqslant e^{-(1+2(r_0/d))}$ for all $j$. Note that $\lambda_j(\Delta)=\lambda_j(\Sigma^{-1}_{\textbf{W}}\Sigma_{\textbf{W}^{*}})$ holds as $\textbf{A}\textbf{B}$ and $\textbf{B}\textbf{A}$ have same set of positive eigen values. Now by \autoref{minimax_lemma} we have that
$$e^{-(1+2(r_0/d))} \leqslant \lambda_{\text{min}}(\Sigma^{-1}_{\textbf{W}}\Sigma_{\textbf{W}^{*}})\leqslant \lambda_{\text{min}}(\Sigma^{-1}_{\textbf{W}})\lambda_{\text{max}}(\Sigma^*_{\textbf{W}}).$$
Also, the other condition implies
$$\lambda_{\text{min}}(\Sigma^{-1}_{\textbf{W}})\|\mu_{\textbf{W}}-\mu^*_{\textbf{W}}\|^2\leqslant 2r_0\implies \|\mu_{\textbf{W}}-\mu^*_{\textbf{W}}\|^2\leqslant 2r_0 e^{(1+2(r_0/d))}\lambda_{\text{max}}(\Sigma^*_{\textbf{W}}).$$
\begin{align*}
\|\Sigma_\textbf{W}\|_{\text{op}}&=\lambda_{\text{max}}\left(\Sigma_{\textbf{W}}\right)\\
&= \dfrac{1}{\lambda_{\text{min}}\left(\Sigma^{-1}_{\textbf{W}}\right)}\\
&\leqslant  \|\Sigma^*_{\textbf{W}}\|_{\text{op}}e^{1+2(r_0/d)}\\
&\leqslant \|\Sigma^*_{\textbf{W}}\|e^{1+2(r_0/d)}\quad\text{(using \autoref{psd_norm_lemma})}.
\end{align*}
The upper bounds on $\|\Sigma_{\textbf{Z}}\|_{\text{op}}$ and $\|\mu_{\textbf{Z}}-\mu^*_{\textbf{Z}}\|^2$ follow analogously.
\end{proof}
\begin{Lemma}\label{control_lemma}
Let $r_0>0$ be given. For any $(q_{\textbf{W}},q_{\textbf{Z}})\in \mathcal{B}^*_{\textbf{W}}(r_0)\times \mathcal{B}^*_{\textbf{Z}}(r_0)$, we have the following inequalities
$$\|\mu_{\textbf{W}}\|\leqslant \|\mu^*_{\textbf{W}}\|+\|\Sigma_{\textbf{W}}\|^{1/2}_{\text{op}}G_1(r_0) \text{ and }\|\mu_{\textbf{Z}}\|\leqslant \|\mu^*_{\textbf{Z}}\|+\|\Sigma_{\textbf{Z}}\|^{1/2}_{\text{op}}G_2(r_0),$$
where $G_1,G_2$ are strictly positive functions such that $\lim\limits_{r\to 0+}G_1(r)=\lim\limits_{r\to 0+}G_2(r)=0$.
\end{Lemma}
\begin{proof}
 The claims follow from \autoref{subset_inclusion} and triangle inequality:
 $$\|\mu_{\textbf{W}}\|\leqslant \|\mu^*_{\textbf{W}}\|+\|\mu_{\textbf{W}}-\mu^*_{\textbf{W}}\|\leqslant \|\mu^*_{\textbf{W}}\|+\|\Sigma_{\textbf{W}}\|^{1/2}_{\text{op}}G_1(r_0),$$
$$\|\mu_{\textbf{Z}}\|\leqslant \|\mu^*_{\textbf{Z}}\|+\|\mu_{\textbf{Z}}-\mu^*_{\textbf{Z}}\|\leqslant \|\mu^*_{\textbf{Z}}\|+\|\Sigma_{\textbf{Z}}\|^{1/2}_{\text{op}}G_2(r_0),$$
with $G_1(r)= \sqrt{ 2r e^{(1+2(r/d))}}$ and $G_2(r)= \sqrt{ 2r e^{(1+2(r/n))}}.$
 \end{proof}
\begin{Lemma}\label{minimax_lemma}
Let $\textbf{A}$ and $\textbf{B}$ are two $k\times k$ symmetric positive definite matrices. We have that 
$$\lambda_{\text{min}}(\textbf{A}\textbf{B})\leqslant \lambda_{\text{min}}(\textbf{A})\lambda_{\text{max}}(\textbf{B}).$$
\end{Lemma}
\begin{proof}
Note that the matrices $\mathbf{B}^{1/2},\mathbf{B}^{1/2}\mathbf{A}\mathbf{B}^{1/2}$ are symmetric. By Courant-Fischer-Weyl minimax principle we have that
\begin{align*}
\lambda_{\text{min}}(\mathbf{A}\mathbf{B}) 
&= \lambda_{\text{min}}(\mathbf{B}^{1/2}\mathbf{A}\mathbf{B}^{1/2}) \\
&= \min\limits_{\substack{\mathcal{F}\subset \mathbb{R}^k\\\text{dim}(\mathcal{F})=1}}
\left(\max_{x\in\mathcal{F}\setminus\{0\}}
\dfrac{\langle\textbf{B}^{1/2}\textbf{A}\textbf{B}^{1/2}x,x\rangle}{\langle x,x \rangle}\right) \\
&= \min\limits_{\substack{\mathcal{F}\subset \mathbb{R}^k\\\text{dim}(\mathcal{F})=1}}
\left(\max_{x\in\mathcal{F}\setminus\{0\}}
\dfrac{\langle\textbf{B}^{1/2}\textbf{A}\textbf{B}^{1/2}x,x\rangle}{\langle \textbf{B}^{1/2}x,\textbf{B}^{1/2}x \rangle}
\dfrac{\langle\textbf{B}x,x\rangle}{\langle x,x \rangle}\right).
\end{align*}

From above we readily obtain
$$\lambda_{\text{min}}(\mathbf{A}\mathbf{B})\leqslant \lambda_{\text{max}}(\textbf{B})\min\limits_{\substack{\mathcal{F}\subset \mathbb{R}^k\\\text{dim}(\mathcal{F})=1}}\left(\max_{x\in\mathcal{F}\setminus\{0\}}\dfrac{\langle\textbf{B}^{1/2}\textbf{A}\textbf{B}^{1/2}x,x\rangle}{\langle \textbf{B}^{1/2}x,\textbf{B}^{1/2}x \rangle}\right).$$
Let $\mathcal{F}=\text{span}(\textbf{B}^{-1/2}v_1)$ where $v_1$ is an unit norm eigen vector of \textbf{A} which corresponds to $\lambda_{\text{min}}(\textbf{A})$. Evaluating on this subspace yields
\begin{align*}
\lambda_{\text{min}}(\mathbf{A}\mathbf{B})&\leqslant\lambda_{\text{max}}(\textbf{B})\left(\max_{\ell\in\mathbb{R}\setminus\{0\}}\dfrac{\langle\textbf{B}^{1/2}\textbf{A}\ell v_1,\textbf{B}^{-1/2}\ell v_1\rangle}{\langle \ell v_1,\ell v_1\rangle}\right)\\
&= \lambda_{\text{max}}(\textbf{B})\left(\dfrac{\langle\textbf{B}^{1/2}\textbf{A} v_1,\textbf{B}^{-1/2} v_1\rangle}{\langle  v_1,v_1\rangle}\right)\\
&= \lambda_{\text{max}}(\textbf{B})\lambda_{\text{min}}(\textbf{A})\left(\dfrac{\langle\textbf{B}^{1/2}v_1,\textbf{B}^{-1/2} v_1\rangle}{\langle  v_1,v_1\rangle}\right)\\
&= \lambda_{\text{min}}(\textbf{A})\lambda_{\text{max}}(\textbf{B}).
\end{align*}
\end{proof}
\begin{Lemma}\label{florians_lemma}
For two $k\times k$ symmetric positive definite matrices $\textbf{A}$ and $\textbf{B}$ we have that
$$\text{tr}\left(\textbf{A}^{-1}\textbf{B}\right)+\text{tr}\left(\textbf{B}^{-1}\textbf{A}\right)-2k\geqslant \dfrac{1}{\lambda_{\text{max}}(\textbf{A})\lambda_{\text{max}}(\textbf{B})}\|\textbf{A}-\textbf{B}\|^2.$$
\end{Lemma}
\begin{proof}
Let us denote
$$\phi(\textbf{A},\textbf{B})= \text{tr}\left(\textbf{A}^{-1}\textbf{B}\right)+\text{tr}\left(\textbf{B}^{-1}\textbf{A}\right)-2k.$$
We have that
\begin{align*}
\phi(\textbf{A},\textbf{B}) &= \text{tr}(\textbf{A}^{-1}\textbf{B}-I_k) 
+ \text{tr}(\textbf{B}^{-1}\textbf{A}-I_k)  
= \text{tr}(\textbf{A}^{-1}(\textbf{B}-\textbf{A}))  
+ \text{tr}(\textbf{B}^{-1}(\textbf{A}-\textbf{B})) \\
&\hspace{51mm}= \text{tr}\left((\textbf{A}^{-1}-\textbf{B}^{-1})(\textbf{B}-\textbf{A})\right),
\end{align*}
but note that since
$$\textbf{A}^{-1}-\textbf{B}^{-1}= \textbf{A}^{-1}\textbf{B}\textbf{B}^{-1}-\textbf{A}^{-1}\textbf{A}\textbf{B}^{-1}=\textbf{A}^{-1}(\textbf{B}-\textbf{A})\textbf{B}^{-1},$$
we have that 
$$\phi(\textbf{A},\textbf{B})= \text{tr}\left(\textbf{A}^{-1}(\textbf{B}-\textbf{A})\textbf{B}^{-1}(\textbf{B}-\textbf{A})\right).$$
Observing that $\textbf{M}=(\textbf{B}-\textbf{A})\textbf{B}^{-1}(\textbf{B}-\textbf{A})=(\textbf{B}-\textbf{A})'\textbf{B}^{-1}(\textbf{B}-\textbf{A})$, we clearly have that \textbf{M} is a positive definite (and symmetric) matrix. Thus, we have that
\begin{align*}
\phi(\textbf{A},\textbf{B}) 
&\geqslant \dfrac{1}{\lambda_{\text{max}}(\textbf{A})}
\text{tr}\left((\textbf{B}-\textbf{A})\textbf{B}^{-1}(\textbf{B}-\textbf{A})\right) \\
&= \dfrac{1}{\lambda_{\text{max}}(\textbf{A})}
\text{tr}\left(\textbf{B}^{-1}(\textbf{B}-\textbf{A})^2\right) \\
&\geqslant \dfrac{1}{\lambda_{\text{max}}(\textbf{A})\lambda_{\text{max}}(\textbf{B})}
\text{tr}\left((\textbf{B}-\textbf{A})^2\right)\quad(\text{using \autoref{lower_bound_on_trace_product}}).
\end{align*}
The conclusion follows since $\textbf{B}-\textbf{A}$ is symmetric, and therefore $\text{tr}((\textbf{B}-\textbf{A})^2)=\text{tr}((\textbf{B}-\textbf{A})(\textbf{B}-\textbf{A})')=\|\textbf{B}-\textbf{A}\|^2$.
\end{proof}
\begin{Lemma}\label{lower_bound_on_trace_product}
For two $k\times k$ symmetric positive definite matrices \textbf{A} and \textbf{B} we have that
$$\text{tr}\left(\textbf{A}\textbf{B}\right)\geqslant \lambda_{\text{min}}\left(\textbf{A}\right)\text{tr}\left(\textbf{B}\right).$$
\end{Lemma}
\begin{proof}
The matrix \(\textbf{C} = \textbf{A}-\lambda_{\text{min}}(\textbf{A})I_k\) is positive semi definite. Thus, there exist matrices \(\textbf{U}, \textbf{V}\) such that \(\textbf{C} = \textbf{U}\textbf{U}'\) and \(\textbf{B} = \textbf{V}\textbf{V}'\), and consequently,
$$\text{tr}(\textbf{C}\textbf{B}) = \text{tr}(\textbf{U}\textbf{U}'\textbf{V}\textbf{V}') = \text{tr}(\textbf{U}'\textbf{V}\textbf{V}'\textbf{U}) = \text{tr}\left((\textbf{U}'\textbf{V})(\textbf{U}'\textbf{V})'\right)\geqslant 0.$$
Therefore,
$$\text{tr}\left((\textbf{A}-\lambda_{\text{min}}(\textbf{A})I_k)\textbf{B}\right)\geqslant 0 \implies \text{tr}\left(\textbf{A}\textbf{B}\right)\geqslant \lambda_{\text{min}}\left(\textbf{A}\right)\text{tr}\left(\textbf{B}\right).$$
\end{proof}
\begin{Lemma}\label{trace_lemma}
    For any two matrices \(\textbf{A}\) and \(\textbf{B}\) compatible for multiplication such that \(\textbf{AB}\) is a $d\times d$ square matrix, we have that
    $$\text{tr}(\textbf{AB})\leqslant \|\textbf{A}\|\cdot\|\textbf{B}\|.$$
\end{Lemma}
\begin{proof}
Note that if \textbf{A} and \textbf{B} were square matrices the result would have been a simple application of Cauchy-Schwarz inequality. Suppose the sizes of the matrices \textbf{A} and \textbf{B} are $d\times n$ and $n\times d$, respectively. Let $\textbf{A}=(A_{ij})$ and $\textbf{B}=(B_{ij})$. We have that
\begin{align*}
|\text{tr}(\textbf{AB})|&=|\langle\textbf{A},\textbf{B}'\rangle_{F}|\quad\left(\langle \cdot,\cdot\rangle_F \text{ is the Frobenious inner product}\right)\\
&\leqslant \|\textbf{A}\|\cdot\|\textbf{B}'\|\quad(\text{Cauchy-Schwarz for inner product spaces})\\
&=\|\textbf{A}\|\cdot\|\textbf{B}\|.
\end{align*}
\end{proof}
\begin{Lemma}\label{expectation_lemmas}
We have that 
$$\mathbb{E}[\text{tr}\left(\textbf{W}\textbf{Z}'\textbf{Z}\textbf{W}'\right)]= \text{tr}\left(\Gamma_{\textbf{W}}\Gamma_\textbf{Z}\right)\quad\text{and}\quad\mathbb{E}[\textbf{W}\Lambda\textbf{W}']=d\text{tr}\left(\Lambda \Sigma_{\textbf{W}}\right)+\text{tr}\left(\mu_\textbf{W}\Lambda\mu_{\textbf{W}}'\right),$$
where the independent random variables $\textbf{W}$ and $\textbf{Z}$ are such that $\textbf{W}\sim \mathcal{N}( \mu_\textbf{W},I_d\otimes \Sigma_{\textbf{W}})$ and $\textbf{Z}\sim \mathcal{N}( \mu_\textbf{Z}, I_n\otimes \Sigma_{\textbf{Z}})$.
\end{Lemma}
\begin{proof}
Recall the following quantities for convenience
$$\Gamma_{\textbf{W}}=d\Sigma_{\textbf{W}}+\mu_{\textbf{W}}'\mu_{\textbf{W}}\text{  and  }\Gamma_{\textbf{Z}}=n\Sigma_{\textbf{Z}}+\mu_{\textbf{Z}}'\mu_{\textbf{Z}}.$$
We compute the above expectations using \autoref{lem} as follows:
\begin{align*}
  \mathbb{E}[\text{tr}\left(\textbf{W}\textbf{Z}'\textbf{Z}\textbf{W}'\right)]&=  \mathbb{E}[\text{tr}\left(\textbf{Z}\textbf{W}'\textbf{W}\textbf{Z}'\right)]\\
  &=\mathbb{E}\left[\text{tr}\left(\mathbb{E}\left[\textbf{Z}\textbf{W}'\textbf{W}\textbf{Z}'\right|\textbf{W}]\right)\right]\\
&=\mathbb{E}\left[\text{tr}\left(\text{tr}\left(\textbf{W}'\textbf{W}\Sigma_\textbf{Z}\right)I_n+\mu_{\textbf{Z}}\left(\textbf{W}'\textbf{W}\right)\mu_{\textbf{Z}}'\right)\right]\\
&=\mathbb{E}\left[n\text{tr}\left(\textbf{W}'\textbf{W}\Sigma_\textbf{Z}\right)+\text{tr}\left(\mu_{\textbf{Z}}\left(\textbf{W}'\textbf{W}\right)\mu_{\textbf{Z}}'\right)\right]\\
&= n\text{tr}\left(\mathbb{E}\left[\textbf{W}'\textbf{W}\right]\Sigma_\textbf{Z}\right)+\text{tr}\left(\mu_{\textbf{Z}}\mathbb{E}\left[\textbf{W}'\textbf{W}\right]\mu_{\textbf{Z}}'\right)\\
&= n\text{tr}\left(\Gamma_{\textbf{W}}\Sigma_\textbf{Z}\right)+\text{tr}\left(\mu_{\textbf{Z}}\Gamma_{\textbf{W}}\mu_{\textbf{Z}}'\right)\\
&=n\text{tr}\left(\Gamma_{\textbf{W}}\Sigma_\textbf{Z}\right)+\text{tr}\left(\Gamma_{\textbf{W}}\mu_{\textbf{Z}}'\mu_{\textbf{Z}}\right)\\
&= \text{tr}\left(\Gamma_W\left(n\Sigma_\textbf{Z}+\mu_{\textbf{Z}}'\mu_\textbf{Z}\right)\right)\\
&= \text{tr}\left(\Gamma_\textbf{W}\Gamma_{\textbf{Z}}\right).
\end{align*}
\begin{align*}
  \mathbb{E}[\text{tr}(\textbf{W}\Lambda\textbf{W}')]&= \text{tr}\left(\mathbb{E}\left[\textbf{W}\Lambda\textbf{W}'\right]\right)\\
  &=\text{tr}\left[\text{tr}\left(\Lambda \Sigma_{\textbf{W}}\right)I_d+\mu_\textbf{W}\Lambda\mu_{\textbf{W}}'\right]\\
   &=d\text{tr}\left(\Lambda \Sigma_{\textbf{W}}\right)+\text{tr}\left(\mu_\textbf{W}\Lambda\mu_{\textbf{W}}'\right).
\end{align*}  
\end{proof}
\begin{Lemma}\label{algeraic_lemma}
 For positive real numbers $a_1,a_2,a_3,a_4$ and $b_1,b_2,b_3,b_4$, we have the following inequality
 $$\dfrac{a_1+a_2+a_3+a_4}{b_1+b_2+b_3+b_4}\leqslant \max\left\{\dfrac{a_1}{b_1},\dfrac{a_2}{b_2},\dfrac{a_3}{b_3},\dfrac{a_4}{b_4}\right\}.$$
\end{Lemma}
\begin{proof}
    We first prove the following inequality for positive real numbers $a_1',a_2',b_1',b_2'$ 
    $$\dfrac{a_1'+a_2'}{b_1'+b_2'}\leqslant \max\left\{\dfrac{a_1'}{b_1'},\dfrac{a_2'}{b_2'}\right\}.$$
There are two possible scenarios: either 1) $b_1'/b_2' \leqslant a_1'/a_2'$, or 2) $b_1'/b_2' > a_1'/a_2'$. It can be checked for the first case $(a_1'+a_2')/(b_1'+b_2')\leqslant a_1'/b_1'$, and for the second case $(a_1'+a_2')/(b_1'+b_2')< a_2'/b_2'$ hold.
Let, $a_1' =a_1+a_2 ,a_2' = a_3+a_4$ and $b_1'= b_1+b_2,b_2' =b_3+b_4$. We have the left hand side is equal to
\begin{align*}
\dfrac{a_1'+a_2'}{b_1'+b_2'}&\leqslant\max\left\{\dfrac{a_1'}{b_1'},\dfrac{a_2'}{b_2'}\right\}\\
&= \max\left\{\dfrac{a_1+a_2}{b_1+b_2},\dfrac{a_3+a_4}{b_3+b_4}\right\}\\
&\leqslant \max\left\{\max\left\{\dfrac{a_1}{b_1},\dfrac{a_2}{b_2}\right\},\max\left\{\dfrac{a_3}{b_3},\dfrac{a_4}{b_4}\right\}\right\}\\
&= \max\left\{\dfrac{a_1}{b_1},\dfrac{a_2}{b_2},\dfrac{a_3}{b_3},\dfrac{a_4}{b_4}\right\}.
\end{align*}
\end{proof}
\begin{Lemma}\label{psi_bound}
Let $\psi:(0,\infty)\to (0,\infty)$ be the function defined by $\psi(x)=x-1-\log x$. Let $c>0$ be given, there exist real numbers $\ell_c, u_c$ such that $\ell_c<1<u_c$ and $[\ell_c,u_c]=\{x:\psi(x)\leqslant c\}$. Additionally,  $e^{-(1+c)}\leqslant\ell_c \leqslant e^{-c}.$ 
\end{Lemma}
\begin{proof}
The function $\psi$ decreases monotonically on $(0,1)$ and then increases monotonically on $(1,\infty)$. $\psi(t)=0$ if and only if $t=1$. Therefore, the first part of the claim follows.
Since it must be the case that $\psi(\ell_c)=c$, we have that
$$\log\left(\ell_c\right)-\ell_c=-(1+c)\implies \ell_ce^{-\ell_c}=e^{-(c+1)}\implies \ell_c = e^{\ell_c}e^{-(1+c)}\in (e^{-(1+c)},e^{-c}).$$
\end{proof}
\begin{Lemma}\label{psd_norm_lemma}
For a $d\times d$ positive definite symmetric matrix $\textbf{A}$ we have the following inequality
$$\dfrac{1}{\sqrt{d}}\|\textbf{A}\|\leqslant \|\textbf{A}\|_{\text{op}}\leqslant \|\textbf{A}\|,$$
and the lower and upper bounds cannot be improved.
\end{Lemma}
\begin{proof}
Skipped. Elementary.
\end{proof}
\section{Numerical experiments regarding the convergence of CAVI dynamics when $k=1$}\label{numerical_k1}
We begin by simulating an original dataset $\textbf{X}$ with $n = 100$, $d = 10$, $k = 1$, and $\tau_0 = 100$. All experiments in this appendix and elsewhere were performed using a Mac M1 computer with 16 GB of memory via Jupyter Notebooks. Each row of $\textbf{X}$ is generated independently from the model in \eqref{model}, where  

\[
\textbf{X} \sim \mathcal{N}(\textbf{Z}_0\textbf{W}'_0, \tau_0^{-1}I_{100} \otimes I_{10})
\]

with $\textbf{W}_0 = (1,\ldots,1)'$ and $\textbf{Z}_0 \sim \mathcal{N}(0, I_{100} \otimes 1)$.

We observe that the normalized sequence ${\mu}^{(t)}_{\textbf{Z}}/\|{\mu}^{(t)}_{\textbf{Z}}\|$ rapidly converges to $\mu_1$, typically within the first five iterations, as suggested by \autoref{convergence_of_mu_z_fixed_update} and illustrated in \autoref{fig:convergence_direction} when we run \autoref{algo1} with $\mu_\textbf{Z}^{(0)}$ initialized randomly. Next, we investigate the convergence of the scaling components.

Once the dataset $\textbf{X}$ is prepared, we proceed with the implementation of \autoref{algo1}. Now, we initialize $\mu^{(0)}_{\textbf{Z}}= a^{(0)}\mu_1$, where the parameter $a^{(0)}$ is an arbitrary strictly positive real number. Additionally, we set $\Sigma^{(0)}_{\textbf{Z}} = b^{(0)}$, where $b^{(0)}$ is also strictly positive. The hyper-parameter is set as $\Lambda = 1$, with a tolerance level of $\varepsilon = 10^{-15}$. Our objective is to verify whether the sequence $(a^{(t)}, b^{(t)})$ converges in $\mathbb{R}^2$ as we iterate \autoref{algo1}.  

\textbf{Observation:} For all positive values of $a^{(0)}$, we observe that the sequence $a^{(t)}$ converges to the fixed value $a^* = 10.039223865837567$ after $T$ iterations, where $T$ depends on the initial values $(a^{(0)},b^{(0)})$. The corresponding sequence $b^{(t)}$ converges to $b^* = 0.0009184540276287452$ (see \autoref{fig:ab_scaling_convergence}). It can be numerically verified that the recorded value $(a^*, b^*)$ is indeed a fixed point of $\Phi$. This verification is straightforward since all other parameters are known, and $\lambda_1 = 1098.453$ is easily computed from the data matrix $\textbf{X}$. Moreover, $\lambda_1$ is significantly larger than the largest root $\beta = 0.09998$ of the quadratic equation  

\[
(\lambda_1)^2 \tau_0^2 - \lambda_1 (\Lambda + (d + n) \tau_0) + d n = 0,
\]

implying that the polynomial $P$ has a unique strictly positive root given by $(a^*)^2$.  

It is also noteworthy that for any initialization $\mu^{(0)}_{\textbf{Z}}$, not necessarily in the proposed format $\mu^{(0)}_{\textbf{Z}}= a^{(0)}\mu_1$, we observe convergence of the sequences $\|\mu^{(t)}_{\textbf{Z}}\|$ and $\Sigma^{(t)}_{\textbf{Z}}$ to $a^*$ and $b^*$, respectively. Finally, we verified that $\|\mu^{(t)}_{\textbf{W}}\|$ and $\Sigma^{(t)}_{\textbf{W}}$ also converged, and

\[
(\mu^*_{\textbf{W}},\Sigma^*_{\textbf{W}}) = F(\mu^*_{\textbf{Z}},\Sigma^*_{\textbf{Z}}).
\]
\begin{Remark}
From the theory of discrete dynamical systems, we numerically verified that $(a^*,b^*)$ is \textit{stable} and \textit{attracting} since both the eigenvalues of the Jacobian $D\Phi$ evaluated at this point are strictly less than one in absolute value. Theoretically, this means apart from some pathological choices, for most initializations $(a^{(0)},b^{(0)})$, the iterates $(a^{(t)},b^{(t)})$ will converge to $(a^*,b^*)$, which is consistent with our numerical findings. 
\end{Remark}
\section{Local behavior of the loss function $\Psi$ near $q^*$}\label{ap_localbehaviour}
\begin{Proposition}[Rotational ambiguity]\label{balls_dont_intersect}
Suppose the hyper-parameter $\Lambda=\lambda I_k$ for some $\lambda >0$ and $q^*=(\mu^*_{\textbf{W}},\Sigma^*_{\textbf{W}},\mu^*_{\textbf{Z}},\Sigma^*_{\textbf{Z}})\in \mathcal{M}$ be a solution to the minimization problem in \eqref{problem}. Let $\mathbf{R}$ be any $k\times k$ orthogonal matrix. If we set $q^*{\mathbf{R}} =(\mu^*_{\textbf{W}}\mathbf{R},\mathbf{R}'\Sigma^*_{\textbf{W}}\mathbf{R},\mu^*_{\textbf{Z}}\mathbf{R},\mathbf{R}'\Sigma^*_{\textbf{Z}}\mathbf{R})$, then $q^*\textbf{R}$ is also a solution to the minimization problem in \eqref{problem}. Consequently, the Hessian of the loss function $\Psi$ is singular if $k\geqslant 2$.
\end{Proposition}
\begin{proof}
We have that $\Lambda = \lambda I_k$ for some $\lambda> 0$. We recall the loss function from \eqref{eq:loss_function}
\begin{align*}
\Psi\left(\mu_{\textbf{W}}, \Sigma_{\textbf{W}}, \mu_{\textbf{Z}}, \Sigma_{\textbf{Z}}\right) &= \frac{\tau_0}{2} \left\|\textbf{X}- \mu_{\textbf{Z}} \mu_{\textbf{W}}' \right\|^2 + \frac{1}{2} \operatorname{tr}\left( \mu_{\textbf{W}}' \Lambda \mu_{\textbf{W}} \right) + \frac{1}{2} \operatorname{tr}\left( \mu_{\textbf{Z}}' \mu_{\textbf{Z}} \right) + \frac{d}{2} \operatorname{tr}\left( \Lambda\, \Sigma_{\textbf{W}} \right)  \\
&\quad  + \frac{n}{2} \operatorname{tr}\left( \Sigma_{\textbf{Z}} \right)+ \frac{\tau_0 d n}{2} \operatorname{tr}\left( \Sigma_{\textbf{W}}\, \Sigma_{\textbf{Z}} \right)+ \frac{\tau_0 d}{2} \operatorname{tr}\left( \Sigma_{\textbf{W}}\, \mu_{\textbf{Z}}' \mu_{\textbf{Z}} \right)  \\
&\quad + \frac{\tau_0 n}{2} \operatorname{tr}\left( \mu_{\textbf{W}}' \mu_{\textbf{W}}\, \Sigma_{\textbf{Z}} \right)- \frac{k}{2} \log \det\left( \Sigma_{\textbf{W}} \right) - \frac{k}{2} \log \det\left( \Sigma_{\textbf{Z}} \right)+ \text{cst}.
\end{align*}
Therefore, $\Psi(q^*) = \Psi(q^* \mathbf{R})$, as all the trace terms and Frobenius norms remain invariant under the proposed transformations.
\end{proof}
When \(\Lambda\) is a multiple of the identity matrix, \autoref{balls_dont_intersect} presents an immediate roadblock to proving results such as \autoref{main_theorem_bpca} whenever \(k \geqslant 2\). This arises because arbitrarily close minimizers exist. More formally, for any \(\eta > 0\), there exists a rotation \(\textbf{R}_\eta\) such that \(\text{KL}(q^* \textbf{R}_\eta \| q^*) \leqslant \eta\). In such cases, the singularity of the Hessian of the loss function \(\Psi\) leads to flatness at \(q^*\), a phenomenon that is numerically demonstrated in the next subsection.
\subsection{Numerical Experiment}
In this subsection, we numerically illustrate that the Hessian of \(\Psi\) is flat at \(q^*\) when the prior hyperparameter \(\Lambda\) is proportional to \(I_k\) for $k\geqslant 2$. Conversely, when \(\Lambda\) is not proportional to \(I_k\), the Hessian is non-singular, ensuring that \(q^*\) is an isolated local minimizer. Additionally, It is shown that the Hessian is always non-singular when \(k = 1\), as there is no rotational ambiguity in this case. We focus on the simplest case, $k=2$, although similar claims hold for higher rank models and can be easily verified numerically with our code. Note that the numerical experiments presented here are independent of CAVI, as our focus is on an analytical property of the loss function $\Psi$. Suppose $q^*= (\mu^*_{\textbf{W}},\Sigma^*_{\textbf{W}},\mu^*_{\textbf{Z}},\Sigma^*_{\textbf{Z}})$ is a local minimum (or a saddle point) of $\Psi$. We recall the loss function $\Psi(q^*)$ from \eqref{eq:loss_function} and observe that the choice of model hyper-parameter $\Lambda$ influences $\Psi$ through the trace terms.

Let $\Lambda_1= I_2,\ \Lambda_2=\text{diag}(1,2)$. We begin with the set-up: $n=4$, $d=3$, $k=2, \tau_0=100$ and we set $\Lambda = \Lambda_2$, so that $\Lambda\not\propto I_2$. Each dataset, denoted as $\textbf{X}_n$, is generated independently from the model in \eqref{model} with  
\begin{equation}\label{data_prep}
\textbf{X}_n \sim \mathcal{N}(\textbf{Z}_0\textbf{W}_0', \tau_0^{-1}I_{n} \otimes I_{d}),
\end{equation}
where $\textbf{W}_0 = (1,\dots,1)'$ and $\textbf{Z}_0 \sim \mathcal{N}(0, I_{n} \otimes I_k)$, with $\tau_0 = 100$. We run \autoref{algo1} with the initialization $\mu^{(0)}_{\textbf{Z}}$, an $n \times k$ matrix with all entries set to $0.1$, and $\Sigma^{(0)}_{\textbf{Z}} = I_k$.  
To record the value of $q^*$, we first prepare a dataset $\textbf{X}_4$ as described in \eqref{data_prep} and then run CAVI (\autoref{algo1}). However, this alone was insufficient, as the gradient norm of $\Psi$ at the converged point $q^*_{\text{CAVI}}$ remained relatively large ($\approx 10^{-4}$). To reduce this value, we performed an additional round of Newton’s descent on $\Psi$, initializing from $q^*_{\text{CAVI}}$ and iterating until we reached the nearest local minimizer (or a saddle point) $q^*$, where the gradient norm reduced to $\approx 10^{-14}$. Next, we repeat the same steps with the important difference being that $\Lambda =\Lambda_1$.

It is important to note that Newton’s descent alone was ineffective in locating $q^*$, as it converged extremely slowly. In contrast, the combination of CAVI and Newton’s descent yielded convergence in just a few seconds.

We compute the Hessian $H(\Psi)$ at $q^*$ in both cases using Newton's method, and observe that $H(\Psi)$ becomes singular when $\Lambda = \Lambda_1$, exhibiting a zero eigenvalue with a corresponding eigenvector $v_0$. This was empirically confirmed by the observation that, when $\Lambda \propto I_k$, the smallest eigenvalue (in absolute value) was on the order of machine precision ($\approx 10^{-15}$), whereas in the alternative case it was significantly larger ( $\approx10^{-2}$). This implies that $q^*$ is not isolated when $\Lambda\propto I_k$, as $\Psi$ is flat along the direction of $v_0$, meaning $\Psi(q^*+\alpha v_0)$ remains constant for small perturbations $\alpha$. However, when $\Lambda=\Lambda_2$, i.e., $\Lambda$ is not proportional to $I_2$, $H(\Psi)$ is non-singular, ensuring local convexity of $\Psi$ at $q^*$. Consequently, $q^*$ is an isolated local minimizer.

The local behavior of $\Psi$ differs fundamentally when $k=1$ as $\Lambda$ is a positive real number in this case, eliminating rotational ambiguity in the model. In this case, $H(\Psi)$ is always non-singular, ensuring that $\Psi$ is always locally convex regardless of the choice $\Lambda$.
 The situation is highlighted through the plots in \autoref{fig:psi_visualization}.
\begin{figure}[h]
    \centering
    \begin{subfigure}{0.49\textwidth}
        \centering
        \includegraphics[width=\linewidth]{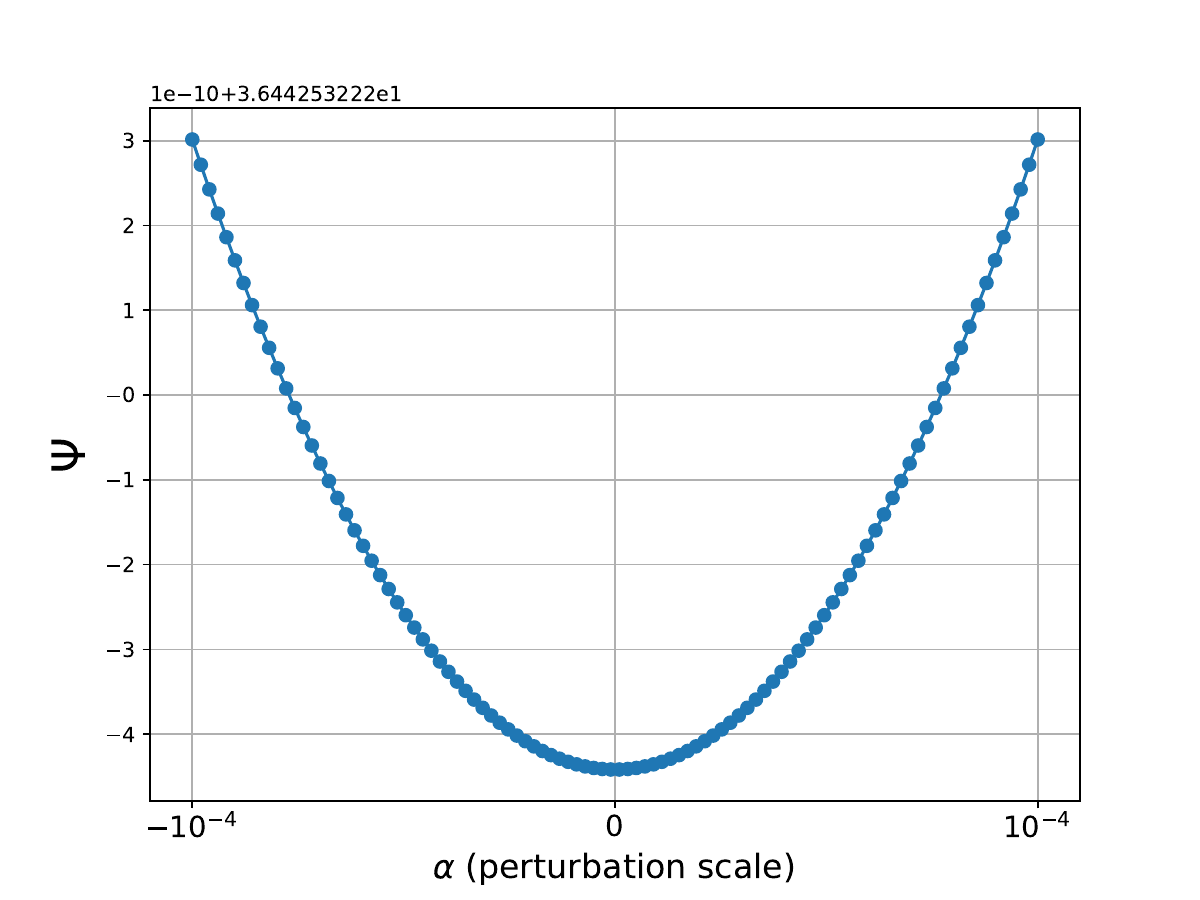}
        \caption{$\Lambda\not\propto I_2,\ k=2$}
    \end{subfigure}
    \hfill
    \begin{subfigure}{0.49\textwidth}
        \centering
        \includegraphics[width=\linewidth]{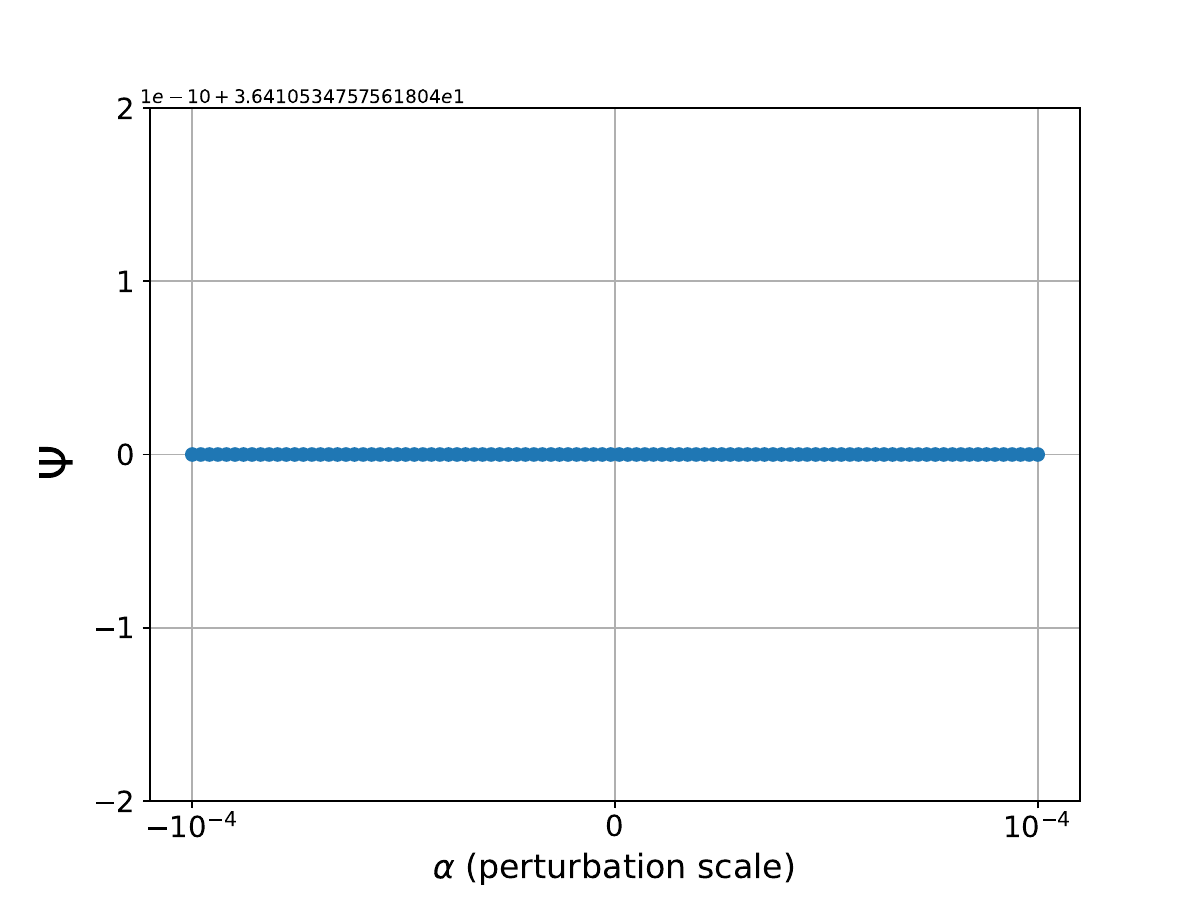}
        \caption{$\Lambda\propto I_2,\ k=2$}
    \end{subfigure}

    \vspace{0.5cm}

    \begin{subfigure}{0.5\textwidth}
        \centering
        \includegraphics[width=\linewidth]{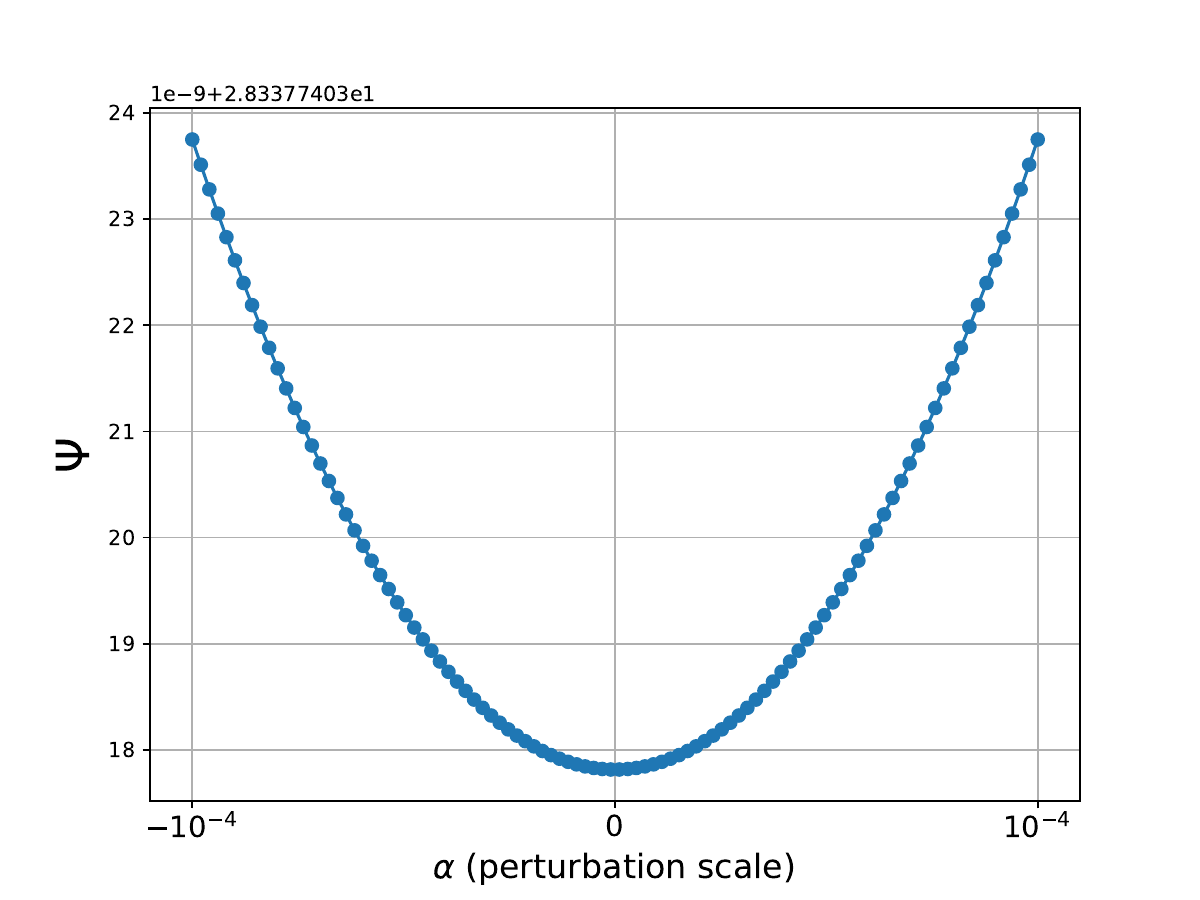}
        \caption{$k=1$}
    \end{subfigure}

    \caption{Visualization of $\Psi$ near $q^*$}
    \label{fig:psi_visualization}
\end{figure}
\end{document}